\documentclass[letterpaper]{article}

\usepackage[usenames,dvipsnames]{xcolor}
\usepackage[export]{adjustbox}

\usepackage{microtype}
\usepackage{graphicx}
\usepackage{enumitem}
\usepackage{subfigure}
\usepackage{booktabs} 
\usepackage{amsmath}
\usepackage{amsthm}
\usepackage{amssymb}
\usepackage{algorithm}
\usepackage{pifont}
\usepackage{algorithmic}
\usepackage{todonotes}
\usepackage{stmaryrd}
\usepackage{tikz}
\usepackage{pgfplots}
\usepackage{dsfont}
\usepackage{dirtytalk}
\pgfplotsset{compat=newest}

\usepackage{hyperref}

\usepackage[nonatbib, final]{neurips_2020}

\newtheorem{theorem}{Theorem}
\newtheorem{prop}{Proposition}

\newtheorem{assumption}{Assumption}

\newtheorem*{theorem*}{Theorem}
\newtheorem*{prop*}{Proposition}

\usepackage{xspace}

\newcommand{\otc}[1]{\textcolor{black}{#1}}
\newcommand{\changede}[1]{\textcolor{black}{#1}}
\newcommand{\changebr}[1]{\textcolor{black}{#1}}
\newcommand{\changelm}[1]{\textcolor{black}{#1}}
\newcommand{\changebrtwo}[1]{\textcolor{black}{#1}}
\newcommand{\changee}[1]{\textcolor{black}{#1}}

 \newcommand{\lints}{\textsc{LinTS}\xspace}
\newcommand{\expfour}{\textsc{Exp}$4$\xspace}
\newcommand{\expthree}{\textsc{Exp}$3$\xspace}

\newcommand{\linucb}{\textsc{LinUCB}\xspace}
\newcommand{\ucb}{\textsc{UCB}\xspace}
\newcommand{\epsgreedy}{\textsc{$\varepsilon$-greedy}\xspace}

\newcommand{\todoeout}[1]{\todo[color=purple]{\scriptsize #1}}

\newcommand{\todol}[1]{\todo[color=GreenYellow]{\small Laurent: #1}}
\newcommand{\interior}[1]{%
  {\kern0pt#1}^{\mathrm{o}}%
}

\newcommand{\cmmnt}[1]{\ignorespaces}

\newcommand{\bookboxx}[1]{\small
\par\medskip\noindent
\framebox[0.99\textwidth]{
\begin{minipage}{0.97\dimexpr\textwidth-\parindent\relax} {#1} \end{minipage} } \par\medskip }

\DeclareMathOperator*{\argmax}{argmax}

\newcommand{\wt}[1]{\widetilde{#1}}
\newcommand{\wh}[1]{\widehat{#1}}

\let\svthefootnote\thefootnote

\newcommand\blankfootnote[1]{%
  \let\thefootnote\relax\footnotetext{#1}%
  \let\thefootnote\svthefootnote%
}

\usepackage{xspace}
\DeclareRobustCommand{\eg}{e.g.,\@\xspace}
\DeclareRobustCommand{\ie}{i.e.,\@\xspace}

\DeclareRobustCommand{\wrt}{w.r.t.\@\xspace}

\title{Adversarial Attacks on Linear Contextual Bandits }
\author{%
  Evrard Garcelon\textsuperscript{$\star$}\\
  Facebook AI Research\\
  \texttt{evrard@fb.com}\\
   \And
   Baptiste Rozi{\`{e}}re\textsuperscript{$\star$} \\
   Facebook AI Research \\
   \texttt{broz@fb.com} \\
   \And
   Laurent Meunier\textsuperscript{$\star$} \\
   Facebook AI Research \\
   \texttt{laurentmeunier@fb.com} \\
   \And
   Jean Tarbouriech \\
   Facebook AI Research \\
   \texttt{jtarbouriech@fb.com} \\
   \And
   Olivier Teytaud \\
   Facebook AI Research\\
   \texttt{oteytaud@fb.com} \\
   \And
   Alessandro Lazaric \\
   Facebook AI Research\\
   \texttt{lazaric@fb.com} \\
   \And
   Matteo Pirotta  \\
   Facebook AI Research\\
   \texttt{pirotta@fb.com} \\
}

\begin{document}
\maketitle
\blankfootnote{\textsuperscript{$\star$} indicates equal contribution}
\begin{abstract}
Contextual bandit algorithms are applied in a wide range of domains, from advertising to recommender systems, from clinical trials to education. In many of these domains, malicious agents may have incentives to \changee{force a bandit algorithm into a desired behavior.} 
For instance, an unscrupulous ad publisher may try to increase their own revenue at the expense of the advertisers; a seller may want to increase the exposure of their products, or thwart a competitor's advertising campaign.
In this paper, we study several attack scenarios and show that a malicious agent can force a linear contextual bandit algorithm to pull any desired arm $T - o(T)$ times over a horizon of $T$ steps, while applying adversarial modifications to either rewards or contexts \changee{with a cumulative cost} that only grow logarithmically as $O(\log T)$.
We also investigate the case when a malicious agent is interested in affecting the behavior of the bandit algorithm in a single context (e.g., a specific user). We first provide sufficient conditions for the feasibility of the attack and 
an efficient algorithm to perform an attack. 
\changee{We empirically validate the proposed approaches in synthetic and real-world datasets.} 
\end{abstract}

\section{Introduction}


Recommender systems are at the heart of the business model of many industries like e-commerce or video streaming~\cite{davidson2010youtube,gomez2015netflix}. The two most common approaches for this task are based either on matrix factorization \cite{park2017comparative} or bandit algorithms \cite{li2010contextual}, which both
rely on a unaltered feedback loop between the recommender system and the user. In recent years, a fair amount of work has been dedicated to understanding how targeted perturbations in the feedback loop can fool a recommender system into recommending low quality items.

Following the line of research on adversarial attacks in supervised learning \cite{biggio2012poisoning,goodfellow2014explaining, jagielski2018manipulating, li2016data, liu2017robust}, attacks on recommender systems have been focused on filtering-based algorithms \cite{10.1145/3298689.3347031, mehta2008attack} and offline contextual bandits \cite{ma2018data}.
The question of adversarial attacks for online bandit algorithms 
has only \cmmnt{started being} \changee{been studied} quite recently \cite{jun2018adversarial, liu2019data, Immorlica2018AdversarialBW, guan2020robust}, and solely in the multi-armed stochastic setting.
Although the idea of online adversarial bandit algorithms is not new (see \expthree algorithm in \cite{auer2002finite}), the focus is different from what we are considering \changee{in this article}. Indeed, algorithms like \expthree or \expfour \cite{lattimore2018bandit} are designed to find optimal actions in hindsight in order to adapt to any rewards stream\cmmnt{ without any further assumptions}.

 The opposition between adversarial and stochastic bandit settings has sparked interests in studying a middle ground.
 In \cite{bubeck2012best}, the learning algorithm has no knowledge of the type of feedback it receives \changee{(either stochastic or adversarial)}. In \cite{lykouris2018stochastic, li2019stochastic, gupta2019better, Lykouris2019CorruptionRE, kapoor2019corruption}, the rewards are assumed to be \changee{corrupted by adversarial rewards}\cmmnt{ but can be perturbed by some attacks}. The authors focus on \changee{building} algorithms able to find the optimal actions even in the presence of some non-random perturbations. \changee{This setting is different from what is studied in this article because} those perturbations are bounded and agnostic to \changee{arms pulled} by the learning algorithm, \changee{i.e., the adversary corrupt the rewards before the algorithm chooses an arm.}

In the broader Deep Reinforcement Learning (DRL) literature, the focus is placed on modifying the observations of different states to fool a DRL system at inference time
\cite{hussenot2019targeted, sunstealthy} \changee{or the rewards \cite{ma2019policy}.}
\paragraph{Contribution.}In this work, we first follow the research direction opened by \cite{jun2018adversarial} where the attacker has the objective of fooling a learning algorithm into taking a specific action as much as possible. \changee{For example} \cmmnt{Consider } in a news recommendation problem, as described in \cite{li2010contextual}, a bandit algorithm chooses between $K$ articles to recommend to a user, based on some information about them, \changee{called} context. We assume that an attacker sits between the user and the website, they can choose the reward (i.e., click or not) for the recommended article observed by the recommending algorithm. Their goal is to fool the bandit algorithm into recommending \changee{some articles} \cmmnt{\changelm{a particular or a set of target articles}} to most users. The contributions of our work can be summarized as follows:
\begin{itemize}
    \item We extend the work of~\cite{jun2018adversarial, liu2019data} to the contextual linear bandit setting showing how to perturb rewards for both stochastic and adversarial algorithms, forcing \textbf{any} bandit algorithms to pull a specific set of arms, $o(T)$ times for logarithmic cost for the attacker.
    \item We analyze, for the first time, the setting in which the attacker can only modify the context $x$ associated with the current user (the reward is not altered). The goal of the attacker is to fool the bandit algorithm into pulling arms of a target set for most users (\ie contexts) while minimizing the total norm of their attacks. We show that the widely known \linucb algorithm \cite{abbasi2011improved, lattimore2018bandit} is vulnerable to this new type of attack.
    \item We present a harder setting for the attacker, where the latter can only modify the context associated to a specific user. This situation may occur when a malicious agent has infected some computers with a Remote Access Trojan (RAT). The attacker can then modify the history of navigation of a specific user and, as a consequence, the information seen by the online recommender system.We show how the attacker can attack the two very common bandit algorithms \linucb and Linear Thompson Sampling (\lints) \cite{agrawal2013thompson,abeille2017linear} and, in certain cases, force them to pull a set of arms most of the time \changebr{when} a specific context (\ie user) is presented to the algorithm (\ie visits a website). 
\end{itemize}

\section{Preliminaries}\label{sec:preliminaries}
We consider the standard contextual linear bandit setting with $K\in \mathbb{N}$ arms. At each time $t$, the agent observes a context $x_{t}\in\mathbb{R}^{d}$, selects an action $a_{t}\in \llbracket 1, K\rrbracket$ and observes a reward: $r_{t,a_{t}} = \langle \theta_{a_{t}}, x_{t}\rangle + \eta_{a_{t}}^{t}$ where for each arm $a$, $\theta_{a}\in \mathbb{R}^{d}$ is a feature vector and $\eta_{a_{t}}^{t}$ is a conditionally independent zero-mean, $\sigma^{2}$-subgaussian noise.  \changee{The contexts are assumed to be sampled \textit{stochastically} except in App.~\ref{app:adversarial_rewards}.} \cmmnt{We also make the following assumptions on the contexts and parameter vectors.}
\begin{assumption}\label{assumption1}
There exist $L>0$ and $\mathcal{D}\subset \mathbb{R}^{d}$, such that for all $t$, $x_{t}\in\mathcal{D}$ and, $\forall x\in\mathcal{D},\forall a\in\llbracket 1, K\rrbracket,~ \|x\|_{2} \leq L \text{ and } \langle \theta_{a}, x\rangle \in (0,1]$. In addition, we assume that there exists $S>0$ such that $\|\theta_{a}\|_{2}\leq S$ for all arms $a$.
\end{assumption}
The agent minimizes the cumulative regret after $T$ steps $R_{T} = \sum_{t=1}^{T} \langle \theta_{a^{\star}_{t}}, x_{t}\rangle - \langle \theta_{a_{t}}, x_{t}\rangle$,
where $a_{t}^{\star} := \argmax_{a} \langle \theta_{a}, x_{t}\rangle$.
A bandit learning algorithm $\mathfrak{A}$ is said to be \emph{no-regret} when it satisfies $R_{T} = o(T)$, i.e., the average expected reward received by $\mathfrak{A}$ \changebr{converges} to the optimal one. Classical bandit algorithms (\eg \linucb and \lints) compute an estimate of the unknown parameters $\theta_a$ using past observations. Formally, for each arm $a \in [K]$ we define $S_a^t$ as the set of times up to $t-1$ (included) where the agent played arm $a$. Then, the estimated parameters are obtained through regularized least-squares regression as $\wh{\theta}_a^t = (X_{t,a} X_{t,a}^\top + \lambda I)^{-1} X_{t,a} Y_{t,a}$, where $\lambda > 0$, $X_{t,a} = (x_i)_{i \in S_a^t} \in \mathbb{R}^{d \times |S_a^t|}$ and $Y_{t,a} = (r_{i,a_i})_{i \in S_a^t} \in \mathbb{R}^{|S_a^t|}$.
Denote by $V_{t,a} = \lambda I + X_{t,a} X_{t,a}^\top$ the design matrix of the regularized least-square problem and by $\|x\|_{V} = \sqrt{x^\top V x}$ the weighted norm \wrt any positive matrix $V \in \mathbb{R}^{d \times d}$.
We define the confidence set:
\begin{equation}
\label{eq:confidence.intervals}
\mathcal{C}_{t,a} = \Big\{ \theta \in \mathbb{R}^d \,:\, \big\|\theta - \widehat{\theta}_{t,a} \big\|_{V_{t,a}} \leq \beta_{t,a} \Big\}
\end{equation}
 where 
        $\beta_{t,a} = \sigma\sqrt{d\log\big( (1 + L^2(1+|S_a^t|)/\lambda)/\delta \big)} + S\sqrt{\lambda},$
which guarantees that $\theta_{a}\in \mathcal{C}_{t,a}$, for all $t>0$, w.p.\ $1-\delta$.
This uncertainty is used to balance the exploration-exploitation trade-off either through optimism (\eg \linucb) or through randomization (\eg \lints).


\section{Online Adversarial Attacks on Rewards}\label{sec:attacks_rewards}
The ultimate goal of a malicious agent is to force a bandit algorithm to perform a desired behavior. An attacker may simply want to induce the bandit algorithm to perform poorly\textemdash ruining the users' experience\textemdash or to force the algorithm to suggest a specific arm. The latter case is particularly interesting in advertising where a seller may want to increase the exposure of its product at the expense of the competitors. Note that the users' experience is also compromised by the latter attack since \changebr{ the suggestions they will receive will not be} tailored to their needs.
Similarly to~\cite{liu2019data,jun2018adversarial}, we focus on the latter objective, \ie to fool the bandit algorithm into pulling arms \changelm{ in $A^\dagger$, a set of target arms,} for $T-o(T)$ time steps (\emph{independently of the user}).

A way to obtain this behavior is to dynamically modify the reward in order to make the bandit algorithm believe that $a^\dagger$ is optimal, \changelm{for some $a^\dagger\in A^\dagger$.}
Clearly, the attacker has to pay a price in order to modify the perceived bandit problem and fool the algorithm.
If there is no restriction on when and how the attacker can alter the reward, the attacker can easily fool the algorithm.
However, this setting is not interesting since the attacker may pay a cost higher than the loss suffered by the attacked algorithm. An attack strategy is considered successful when the total cost of the attack is sublinear in $T$.

In this section, we show that under Assumption~\ref{assumption1}, there exists an attack algorithm that is successful against any bandit algorithm, stochastic or adverserial.

\textbf{Setting.} 
We assume that the attacker has the same knowledge as the bandit algorithm $\mathfrak{A}$ about the problem (\ie knows $\sigma$ and $L$). 
The attacker is assumed to be able to observe the context $x_t$, the arm $a_t$ pulled by $\mathfrak{A}$, and can modify the reward received by $\mathfrak{A}$.
When the attacker modifies the reward $r_{t, a_{t}}$ into $\wt{r}_{t, a_{t}}$ the \emph{instantaneous cost} of the attack is defined as $c_{t} :=\big| r_{t,a_{t}} - \wt{r}_{t, a_{t}}\big|$. The goal of the attacker is to fool algorithm $\mathfrak{A}$ such that \changelm{the arms in $A^\dagger$ are} pulled $T - o(T)$ times and $\sum_{t=1}^{T}  c_{t} = o(T)$. \changee{We also assume that the action for the arms in the target set is strictly positive for every context $x\in\mathcal{D}$. That is to say that $\Delta := \min_{x\in \mathcal{D}}\left\{ \langle x, \theta_{a_{\star}^{\dagger}(x)}\rangle - \max_{a\in A
^{\dagger}, a\neq a_{\star}^{\dagger}(x)} \langle x, \theta_{a}\rangle\right\}>0$ where $a_{\star}^{\dagger}(x) = \arg\max_{a\in A^{\dagger}} \langle x, \theta_{a}\rangle$ for every $x\in \mathcal{D}$}.

\textbf{Attack idea.} We leverage the idea presented in \cite{liu2019data} and \cite{jun2018adversarial} where the attacker lowers the reward of arms $a\notin A^{\dagger}$ so that algorithm $\mathfrak{A}$ learns that an arm of the target set is optimal for every context.
Since $\mathfrak{A}$ is assumed to be no-regret, the attacker only needs to modify the rewards $o(T)$ times to achieve this goal.
%
Lowering the rewards has the effect of shifting
the vectors $(\theta_{a})_{a\notin A^{\dagger}}$ to new vectors $(\theta'_{a})_{a\notin A^{\dagger}}$ such that for all arms $a\notin A^{\dagger}$ and all contexts $x\in\mathcal{D}$, there exists an arm $a^\dagger\in A^\dagger$  such that $\langle\theta'_{a}, x\rangle \leq \langle \theta_{a^{\dagger}}, x\rangle$. Since rewards are assumed to be bounded (see Asm.~\ref{assumption1}), this objective can be achieved by simply forcing the reward of non-target arms $a\notin A^\dagger$ to the minimum value.
Contextual ACE (see Fig.~\ref{alg:context_attack_protocol}) implements a soft version of this idea by leveraging the knowledge of the reward distribution.
At each round $t$, Contextual ACE modifies the reward perceived by $\mathfrak{A}$ as follows:
\vspace{-0.12cm}
 \begin{equation}
         \label{eq:perturbed.reward2}
 \widetilde{r}^{1}_{t,a_{t}} =\eta_{t}'\mathds{1}_{\{a_{t} \notin A^{\dagger}\}}+r_{t, a_t}\mathds{1}_{\{a_{t} \in A^{\dagger}\}} 
 \end{equation}

where $\eta_{t}'$ is a $\sigma$-subgaussian random variable generated by the attacker independently of all other random variables. Contextual ACE transforms the original problem into a \emph{stationary} bandit problem in which there is a targeted arm that is optimal for all contexts and all non targeted arms have expected reward of $0$. \changee{The following propostion shows that the cumulative cost of the attack is sublinear.}

\begin{prop}\label{prop:reward_attack}
	For any $\delta\in(0, 1/K]$, when using Contextual ACE algorithm (Fig. ~\ref{alg:attacker_rewards}) with perturbed rewards $\wt{r}^{1}$, with probability at least $1-K\delta$, algorithm $\mathfrak{A}$ pulls \changelm{an arm in $A^{\dagger}$} for $T - o(T)$ time steps and the total cost of attacks is $o(T)$.
\end{prop}
The proof of this proposition is provided in App.~\ref{app:proof_prop_rewd_attack}. 
While Prop.~\ref{prop:reward_attack} holds for any no-regret algorithm $\mathfrak{A}$, we can provide a more precise bound on the total cost by inspecting the algorithm.
For example, we can show (see App.~\ref{app:algorithms}), that, with probability at least $1-K\delta$, \changebr{the number of times} \linucb~\cite{abbasi2011improved} pulls arms not in \changelm{$A^\dagger$ is at most $\sum_{j\notin A^{\dagger}} N_{j}(T) \leq  \frac{64K\sigma^{2}\lambda S^{2}}{\Delta^{2}}\Big( d\log\Big(\frac{\lambda + \frac{TL^{2}}{d}}{\delta^{2}}\Big) \Big)^{2}$} . 
This directly translates \changebr{into} a bound on the total cost. 

\textbf{Comparison with ACE \cite{liu2019data}.} In the stochastic setting, the ACE algorithm~\cite{liu2019data} leverages a bound on the expected reward of each arm in order to modify the reward. 
However, the perturbed reward process seen by algorithm $\mathfrak{A}$ is non-stationary and in general there is no guarantee that an algorithm minimizing the regret in a stationary bandit problem keeps the same performance when the bandit problem is not stationary anymore. Nonetheless, transposing the idea of the ACE algorithm to our setting would give an attack of the following form, where at time $t$, Alg. $\mathfrak{A}$ pulls arm $a_{t}$ and receives rewards \changebr{$\wt{r}^{2}_{t,a_{t}}$}: 
\vspace{-0.2cm}
\begin{align*}
        \wt{r}^{2}_{t, a_{t}} =
                (r_{t, a_{t}} + \max(-1, \min(0, C_{t, a_{t}}))) \mathds{1}_{\{a_t \notin A^\dagger\}} + 
                r_{t, a_t} \mathds{1}_{\{a_t \in A^\dagger\}}
\end{align*}
with $C_{t,a_{t}} = (1 - \gamma)\min_{a^\dagger\in A^\dagger}\min_{\theta \in \mathcal{C}_{t,a^{\dagger}}} \left\langle \theta, x_{t} \right\rangle - \max_{\theta\in\mathcal{C}_{t,a_t}} \left\langle \theta, x_{t}\right\rangle$.
Note that $\mathcal{C}_{t,a}$ is defined as in Eq.~\ref{eq:confidence.intervals} using the \emph{non-perturbed} rewards, \ie $Y_{t,a} = (r_{i,a_i})_{i \in S_a^t}$. 

\textbf{Bounded Rewards.}  The bounded reward assumption is necessary in our analysis to prove a formal bound on the total cost of the attacks for \textit{any} no-regret bandit algorithm, otherwise we need more information about the attacked algorithm. In practice, the second attack on the rewards, $\wt{r}^{2}$, can be used in the case of unbounded rewards for any
algorithms. The difficulty for unbounded reward is that the attacker has to adapt to the environment reward but in order to do so the reward process observed by the bandit algorithm becomes non-stationary under the attack. Thus, there is no guarantee that an algorithm like \linucb will pull a target arm as the proof relies on the environment observed by the bandit algorithm being stationary. 
We observe empirically that the total cost of attack is sublinear when using $\wt{r}^{2}$.

\cite{jun2018adversarial} does not assume that rewards are bounded but focus on attacking algorithms in the stochastic multi-armed setting. That is to say they study attacks only designed for $\varepsilon$-greedy and \ucb while we provide an efficient attack for any algorithms in the linear contextual case. We can extend their work, and thus remove the bounded reward assumption, in the linear contextual case by using the following attack, designed only for \linucb:
\begin{align}
    \wt{r}^{3}_{t, a_{t}} = \left(r_{t, a_{t}} + \min_{a^\dagger\in A^\dagger}\min_{\theta \in \mathcal{C}_{t,a^{\dagger}}} \left\langle \theta, x_{t} \right\rangle - \max_{\theta\in\mathcal{C}_{t,a_t}} \left\langle \theta, x_{t}\right\rangle\right) \mathds{1}_{\{a_t \notin A^\dagger\}} + r_{t, a_t} \mathds{1}_{\{a_t \in A^\dagger\}}
\end{align}
with $C_{t,a}$ defined as in Eq.~\eqref{eq:confidence.intervals}. Although, the attack $\wt{r}^{3}$ is not stationary, it is possible to prove that the total cost of attack is $\mathcal{O}(\log(T))$ because we know that the attacked bandit algorithm is \linucb. 

\textbf{Constrained Attack.}
When the attacker has a constraint on the instantaneous cost of \changebr{the} attack, using the perturbed reward $\widetilde{r}^{1}$ may not be possible as the cost of the attack at time $t$ is not decreasing over time. Using the perturbed reward $\widetilde{r}^{2}$ offers a more flexible type of attack with more control on the instantaneous cost thanks to the parameter $\gamma$. \changede{But it still suffers from a minimal cost of attack from lowering rewards of arms not in $A^{\dagger}$.}

\textbf{Defense mechanism.}
The attack based on reward $\wt{r}_1$ is hardly detectable without prior knownledge about the problem.
In fact, the reward process associated to $\wt{r}_1$ is stationary and compatible with the assumption about the true reward (\eg subgaussian). While having very low rewards is reasonable in advertising, it can make the attack easily detectable in some other problems.
On the other hand, the fact that $\wt{r}_2$ is a non-stationary process makes this attack easier to detect.
When some data are already available on \changebr{each arm}, the learner can monitor the difference between the average reward\changebr{s} per action \changebr{computed on new and old data.}


\section{Online Adversarial Attacks on Contexts}
\label{sec:attack_all_context}
In this section, we consider the attacker to be able to alter the context $x_t$ perceived by the algorithm rather than the reward. The attacker is now restricted to change the type of users presented to the learning algorithm $\mathfrak{A}$, hence changing its perception of the environment. We show that under the assumption that the attacker knows a lower-bound to the reward of the target set, it is possible to fool \linucb.

\textbf{Setting.} As in Sec.~\ref{sec:attacks_rewards}, we consider the attacker to have the same knowledge about the problem as $\mathfrak{A}$.
The main difference with \otc{the} previous setting is that the attacker \changee{attacks} before the algorithm. \changee{We adopt a \textit{white-box} \cite{goodfellow2014explaining} setting attacking \linucb.}
The goal of the attacker is unchanged: \changebr{they aim at forcing} the algorithm to pull \changelm{arms in $A^\dagger$ for} $T -o(T)$ \changebr{time steps} while paying a sublinear total cost. 
We denote by $\widetilde{x}_t$ the context after the attack and by $c_t = \|x_t - \widetilde{x}_t\|_2$ the instantaneous cost.

\textbf{Difference between attacks on contexts and rewards.} Perturbing contexts is fundamentally different from perturbing the rewards. The attacker only modifies the context that is shown to the bandit algorithm. The true context, which is used to compute the reward, remains unchanged. In other words, the attacker cannot modify the reward observed by the bandit algorithm. Instead, the attack algorithm described in this  section fools the bandit algorithm by making the rewards appear small relative to the contexts and requires more assumptions on the bandit algorithm than in Sec.~\ref{sec:attacks_rewards}.
\textbf{Attack Idea.} The idea \changee{of} the attack in this setting is similar to the attack \changee{of} Sec.~\ref{sec:attacks_rewards}. The attacker builds a bandit problem where arm an $a^{\dagger}\in A^{\dagger}$ is optimal for all contexts by lowering the perceived value of all other arms not in $A^{\dagger}$. 
The attacker cannot modify the reward but, thanks to the linear reward assumption, they can scale the contexts to decrease the predicted rewards in the original context. 

At time $t$, the attacker receives the context $x_t$ and computes the attack. 
Thanks to the white-box setting, it computes the arm $a_{t}$ that algorithm $\mathfrak{A}$ would pull if presented with context $x_{t}$. If \changelm{$a_{t} \notin A^{\dagger}$} then the attacker changes the context to $\wt{x}_{t} =  \alpha_{a_t} x_{t}$ with $\alpha_{a_t} > \max_{x \in \mathcal{D}}\changelm{\min_{a^\dagger\in A^\dagger}} \langle \theta_{a_t}, x \rangle/\langle \theta_{a^{\dagger}}, x \rangle$.\changee{This factor is chosen such that for a ridge regression computed on the dataset $(\alpha x_{i}, \langle \theta, x_{i} \rangle)_{i}$ outputs a parameter close to $\theta/\alpha$ therefore the attacker needs to choose $\alpha$ such that for every context $x\in\mathcal{D}$, $\langle x, \theta/\alpha\rangle \leq \max_{a^{\dagger}\in A^{\dagger}} \langle x, \theta_{a
^{\dagger}}, x \rangle$.}
In other words, the attacker performs a dilation of the incoming context every time algorithm $\mathfrak{A}$ does not pull \changelm{an arm in $A^{\dagger}$}. The fact that the decision rule used by \linucb is invariant by dilation guarantees that the attacker will not inadvertently lower the perceived rewards for \changelm{arms in $A^{\dagger}$}.
Because the rewards are assumed to be linear, presenting a large context $\alpha x$ and receiving the reward associated with the normal context $x$ will skew the estimated rewards of \linucb. The attack protocol is summarized in Fig.~\ref{alg:context_attack_protocol}.

\changee{In order to compute the parameter $\alpha$ used in the attack, we make the following assumption concerning the performance of the arms in the target set:}
\begin{assumption}\label{assumption2}
	For all $x\in \mathcal{D}$, \changelm{there exists $a^\dagger\in A^\dagger$}, such that $0 <\nu \leq \left\langle x, \theta_{a^{\dagger}} \right\rangle$ and $\nu$ is known to the attacker.
\end{assumption}

\changee{\textbf{Knowing $\nu$.} For advertising and recommendation systems, knowing $\nu$ is not problematic. Indeed in those cases, the reward is the probability of impression of the ad ($r \in [0,1]$). The attacker has the freedom to choose one of multiple target arms with strictly positive click probability in every context. This freedom is an important aspect for the attacker since it allows the attacker to cherry pick the target ad(s). In particular, the attacker can estimate $\nu$ based on data from previous campaigns (only for the target ad). For instance, a company could have run many ad campaigns for one of their products and try to get the defender’s system to advertise it.}

An issue is that the norm of the attacked context can be greater that the upper bound $L$ of Assumption~\ref{assumption1}. To prevent this issue, we choose a context-dependent multiplicative constant $\alpha(x) = \min\{ 2/\nu, L/\|x\|_{2}\}$ which amounts to clip the norm of the attacked context to $L$. In Sec.~\ref{sec:experiments}, we show that this attack is effective for different size of target arms sets. We also show that in the case of contexts such that $\|x\|_{2} \leq \nu L/2$ that the cost of attacks is logarithmic in the horizon $T$.

\begin{figure}[t]
\begin{minipage}{0.45\linewidth}
\bookboxx{
        \noindent \textbf{For} time $t=1, 2, ..., T$ \textbf{do}
        \begin{enumerate}[leftmargin=4mm,itemsep=0mm]
                \item Alg. $\mathfrak{A}$ chooses arm $a_{t}$ based on context $x_{t}$
                \item Environment generates reward: $r_{t,a_{t}} = \langle \theta_{a_{t}}, x_{t}\rangle + \eta_{t}$ with $\eta^{t}_{a_t}$ conditionally $\sigma^{2}$-subgaussian
                \item Attacker observes reward $r_{t,a_{t}}$ and feeds the perturbed reward $\wt{r}^{1}_{t,a_{t}}$ (or $\wt{r}^{2}_{t,a_{t}}$) to $\mathfrak{A}$ 
        \end{enumerate}
}
\vspace{-0.1in}
\caption{\small Contextual ACE algorithm}
\label{alg:attacker_rewards}
\end{minipage}\hfill
\begin{minipage}{0.52\linewidth}
\bookboxx{
        \textbf{Input:} attack parameter: $\alpha$ \\
        \noindent \textbf{For} time $t=1, 2, ..., T$ \textbf{do}
        \begin{enumerate}[leftmargin=4mm,itemsep=0mm]
                \item Attacker observes the context $x_{t}$, computes potential arm $a_{t}'$ and \otc{sets} $\wt{x}_{t} = x_{t} + (\alpha(x_{t}) -1 )x_{t}~\mathds{1}_{\{ a_{t}' \notin  A^{\dagger}\}}$
                \item Alg. $\mathfrak{A}$ chooses arm $a_{t}$ based on context $\wt{x}_{t}$
                \item Environment generates reward: $r_{t,a_{t}} = \langle \theta_{a_{t}}, x_{t}\rangle + \eta_{t}$ with $\eta_{t}$ conditionally $\sigma^{2}$-subgaussian
                \item Alg. $\mathfrak{A}$ observes reward $r_{t,a_{t}}$
        \end{enumerate}
}
\vspace{-0.1in}
\caption{\small ConicAttack algorithm.}
\label{alg:context_attack_protocol}
\end{minipage}
\vspace{-0.15in}
\end{figure}
 


\begin{prop}
\label{prop:cost_attack_all_ctx}
	Using the attack described in Fig.~\ref{alg:context_attack_protocol} \changee{and assuming that $\|x\|_{2}\leq \nu L/2$ for all contexts $x\in\mathcal{D}$}, for any $\delta\in (0,1/K]$, with probability at least $1 - K\delta$, the number of times \linucb does not pull \changebrtwo{an} arm \changelm{in $A^{\dagger}$} \otc{before time $T$} is at most     $\sum_{j\notin A^{\dagger}} N_{j}(T) \leq 32K^{2}\left( \frac{\lambda}{\alpha^{2}} + \sigma^{2}d\log\left(\frac{\lambda d + TL^2\alpha^{2}}{d\lambda\delta}\right) \right)^{3}$ 
	with $N_{j}(T)$ the number of times arm $j$ has been pulled \otc{during the first }$T$ steps, 
The total cost for the attacker is bounded by: $   \sum_{t=1}^{T} c_{t} \leq \frac{64K^{2}}{\nu}\left( \frac{\lambda}{\alpha^{2}} + \sigma^{2}d\log\left(\frac{\lambda d + TL^2\alpha^{2}}{d\lambda\delta}\right) \right)^{3}$ with $\alpha = 2/\nu$. 
\end{prop}

The proof of Proposition \ref{prop:cost_attack_all_ctx} (see App.~\ref{app:proof_attack_all_ctx}) assumes that the attacker can attack at any time step, and that \changebr{they} can know in advance which arm will be pulled by Alg. $\mathfrak{A}$ in a given context. Thus it is not applicable to random exploration algorithms like \lints \cite{agrawal2013thompson} and \epsgreedy.  We also observed empirically that thowe two randomized algorithms are more robust to attacks (see Sec.~\ref{sec:experiments}) than \linucb.

\textbf{Norm Clipping.} Clipping the norm of the attacked contexts is not beneficial for the attacker. Indeed, this means that an attacked context was violating the assumption (used by the bandit algorithm) that contexts are bounded by $L$. The attack could then be easily detectable and may succeed only because it is breaking an underlying assumption used by the bandit algorithm. Prop.~\ref{prop:cost_attack_all_ctx} provides a theoretical grounding for the proposed attack when contexts are bounded by $\nu L/2$ and not only $L$.
Although, we can not prove a bound on the cumulative cost of attacks in general, we show in Sec.~\ref{sec:experiments} that attacks are still successful for multiple datasets where contexts are not bounded by $\nu L/2$. 

\vspace{-.1in}
\section{Offline attacks on a Single Context}\label{sec:attack_one_context}
Previous sections focused on the man-in-the-middle (MITM) attack either on reward or context.
\changebr{The} MITM attack allows the attacker to arbitrarily change the information observed by the recommender system at each round.
\changebr{This attack may be hardly feasible in practice, since the exchange channels are generally protected by authentication and cryptographic systems.} 
In this section, we consider the scenario where the attacker has control over a single user $u$.
As an example, consider the case where the device of the user is infected by a malware (e.g., Trojan horse), giving full control of the system to the malicious agent.
The attacker can thus modify the context of the specific user (e.g., by altering the cookies) that is perceived by the recommender system. 
We believe that changes to the context (\eg cookies) are more subtle and less easily detectable than changes to the reward (\eg click). Moreover, if the reward is a purchase, it cannot be altered easily by taking control of the user's device.
Clearly, the impact of the attacker on the overall performance of the recommender system depends on the frequency of the specific user, that is out of the attacker's control. It may be thus difficult to obtain guarantees on the cumulative regret of algorithm $\mathfrak{A}$.
For this reason, we mainly focus on the study of the feasibility of the attack.

The attacker targets a specific user (i.e., the infected user) associated to a context $x^\dagger$.
Similarly to Sec.~\ref{sec:attack_all_context}, the objective of the attacker is to find the minimal change to the context presented to the recommender system $\mathfrak{A}$ such that $\mathfrak{A}$ \changebrtwo{selects an arm in $A^\dagger$}.
$\mathfrak{A}$ observes a modified context $\wt{x}$ instead of $x^\dagger$. After selecting an arm $a_t$, $\mathfrak{A}$ observes the true noisy reward $r_{t,a_t} = \langle \theta_{a_t}, x^{\dagger}\rangle + \eta^t_{a_t}$.
We still study a white-box setting: the attacker can access all the \changebr{parameters of} $\mathfrak{A}$.

In this section, we show under which condition it is possible for an attacker to fool both an optimistic and posterior sampling algorithm.

\subsection{Optimistic Algorithm: \linucb}
\label{sec:optimistic_algorithms}
We consider the \linucb algorithm which chooses the arm to pull by maximizing an upper-confidence bound on the expected reward.
For each arm $a$ and context $x$, the UCB value is given by $\max_{\theta \in \mathcal{C}_{t,a}}  \langle x, \theta \rangle = \langle x, \hat{\theta}_{a}^t \rangle + \beta_{t,a} \|x \|_{\wt{V}_{t,a}^{-1}}$ (see Sec.~\ref{sec:preliminaries}).
%
The objective of the attacker is to force \linucb to pull \changebrtwo{an arm in $A^\dagger$} once presented with context $x^\dagger$.
This means to find a perturbation of context $x^\dagger$ that makes \changebrtwo{any arm in $A^\dagger$} the most optimistic arm.
Clearly, we would like to keep the perturbation as small as possible to reduce the cost for the attacker and the probability of being detected. Formally, the attacker needs to solve the following \emph{non-convex} optimization problem:
\begin{equation}\label{eq:attack_one_user}
\begin{aligned}
\min_{y\in \mathbb{R}^{d}} \quad & \|y\|_2 \quad \quad \text{s.t }\quad &  \changelm{\max_{a\notin A^\dagger}}\max_{\theta \in \wt{\mathcal{C}}_{t,a}}  \langle x^\dagger + y, \theta \rangle + \xi \leq \changebrtwo{\max_{a^\dagger \in A^\dagger}} \max_{\theta \in \widetilde{\mathcal{C}}_{t,a^\dagger}}  \langle x^\dagger + y, \theta \rangle \\
\end{aligned}
\end{equation}
where $\xi>0$ is a parameter of the attacker and $\wt{\mathcal{C}}_{t,a} := \big\{\theta \mid \|\theta - \hat{\theta}_{a}^t\|_{\wt{V}_{t,a}} \leq \beta_{t,a} \big\}$ is the confidence set constructed by \linucb. We use the notation $\wt{\mathcal{C}}, \widetilde{V}$ to stress the fact that \linucb observes only the modified context.
%
%
%
In contrast to Sec.~\ref{sec:attacks_rewards} \changebr{and}~\ref{sec:attack_all_context}, the attacker may not be able to force the algorithm to pull \changebrtwo{any of the target arms in $A^\dagger$}. In other words, Problem~\ref{eq:attack_one_user} may not be feasible. 
However, we are able to characterize the feasibility of~\eqref{eq:attack_one_user}.
\begin{theorem}\label{thm:feasibility_attack_one_user}
        Problem \eqref{eq:attack_one_user} is feasible at time $t$ \emph{iff.} 
        \begin{equation}\label{eq:feasibilty_condition}
        \exists \theta \in \changelm{\cup_{a^\dagger\in A^{\dagger}}}\wt{\mathcal{C}}_{t, a^{\dagger}}, ~ \theta\not\in \text{Conv}\Big( \cup_{\changelm{a\notin A^{\dagger}}} \wt{\mathcal{C}}_{t,a}\Big)
        \end{equation}
\end{theorem}

The condition given by Theorem \ref{thm:feasibility_attack_one_user} says that this attack can be done when there exists a vector $x$ for which \changebrtwo{an arm in $A^{\dagger}$} is assumed to be optimal according to \linucb. The condition mainly stems from the fact that optimizing a linear product on a convex compact set will reach its maximum on the edge of this set. In our case this set is the convex hull of the confidence ellipsoids of \linucb. Although it is possible to use \otc{an} optimization algorithm for this class of non-convex problems\textemdash \eg DC programming~\cite{tuy1995dc}\textemdash they are still slow compared to convex algorithms. Therefore, we present a simple convex relaxation of the previous problem \changebrtwo{for a single target arm $a^\dagger \in A^\dagger$} that still enjoys some empirical performance compared to Problem \eqref{eq:attack_one_user}. \changebrtwo{The final attack can then be computed as the minimum of the attacks obtained for each $a^\dagger \in A^\dagger$.} The relaxed problem is the following \changebrtwo{for each $a^\dagger\in A^\dagger$}:
\begin{equation}\label{eq:relaxed_attack_one_user}
\begin{aligned}
\min_{y\in \mathbb{R}^{d}} \quad & \| y \|_2 \quad \quad \text{s.t }\quad &  \max_{a\neq a^{\dagger}, \changee{a\not\in A^{\dagger}}} \max_{\theta \in \mathcal{C}_{t,a}} \langle x^{\dagger} + y, \theta - \hat{\theta}_{a^{\dagger}}^t \rangle \leq -\xi
\end{aligned}
\end{equation}
Since the RHS of the constraint in Problem \eqref{eq:attack_one_user} can be written as $\max_{\theta\in\mathcal{C}_{t,a^{\dagger}}} \langle \theta, x^{\dagger} + y \rangle$ for any $y$, the relaxation here consists in using $\langle\theta, x^{\dagger}+y\rangle$ as a lower-bound to this maximum for any $\theta\in\mathcal{C}_{t,a^{\dagger}}$. 

For the relaxed Problem \eqref{eq:relaxed_attack_one_user}, the same type of reasoning as for Problem \eqref{eq:attack_one_user} gives that Problem \eqref{eq:relaxed_attack_one_user} is feasible if and only if $     \hat{\theta}_{a^{\dagger}}(t)\not\in \text{Conv}\left( \bigcup_{a\neq a^{\dagger}, \changee{a\not\in A^{\dagger}}} \mathcal{C}_{t,a}\right)$. 


If Condition \eqref{eq:feasibilty_condition} is not met, \changebrtwo{no arm $a^{\dagger} \in A^\dagger$} can be pulled by \linucb. Indeed, the proof of Theorem \ref{thm:feasibility_attack_one_user} shows that the upper-confidence of \changebrtwo{every arm in $A^{\dagger}$} is always dominated by another arm for any context. In other words, if \changebrtwo{any arm in $A^{\dagger}$} is optimal for some contexts then the condition is satisfied a linear number of times for \linucb (for formal proof of this fact see App.~\ref{app:condition_linear}).

 \subsection{Random Exploration Algorithm: \lints}

The previous subsection focused on \linucb, however we can obtain similar guarantees for  algorithms with random exploration such as \lints. In this case, it is not possible to guarantee that a specific arm will be pulled for a given context because of the randomness in the arm selection process. The objective is to guarantee that \changebrtwo{an arm from $A^{\dagger}$} is pulled with probability at least $1-\delta$.
Similarly to the previous subsection, the problem of the attacker can be written as: 
\begin{equation}\label{eq:TS_attack_one_user}
\begin{aligned}
\min_{y\in \mathbb{R}^{d}} \quad & \| y \| \quad \quad \text{s.t }\quad &  \mathbb{P}\left( \exists {a^\dagger\in A^\dagger},~\forall \changebrtwo{a\not\in A^{\dagger},~} \langle x^{\dagger} + y, \wt{\theta}_{a} - \wt{\theta}_{a^{\dagger}} \rangle \leq - \xi\right) \geq 1 - \delta
\end{aligned}
\end{equation}

where the $\wt{\theta}_{a}$ for different arms $a$ are independently drawn from a normal distribution with mean $\hat{\theta}_{a}(t)$ and covariance matrix $\upsilon^{2}\bar{V}_{a}^{-1}(t)$ with $\upsilon = \sigma\sqrt{9d\ln(T/\delta)}$. Solving this problem is not easy and in general not possible, \changebrtwo{even for a single arm}. For a given $x$ and arm $a$, the random variable $\langle x, \wt{\theta}_{a}\rangle$ is normally distributed with mean $\mu_{a}(x) := \langle \hat{\theta}_{a}(t), x\rangle$ and variance $\sigma_{a}^{2}(x) := \nu^{2}||x||_{\bar{V}_{a}^{-1}(t)}^{2}$. We can then write $\langle x,\wt{\theta}_{a}\rangle = \mu_{a}(x) + \sigma_{a}(x)Z_{a}$ with $(Z_{a})_{a}\sim\mathcal{N}(0, I_{K})$. For \changelm{the sake of} clarity, we drop the variable $x$ when writing $\mu_{a}(x)$ and $\sigma_{a}(x)$.

Let's imagine (just for this paragraph) that $A^{\dagger} = \{ a^{\dagger}\}$, then the constraint in Problem \eqref{eq:TS_attack_one_user} becomes $\changebrtwo{\left[1-\mathbb{E}_{Z_{a^{\dagger}}}\left( \Pi_{\changebrtwo{a\not\in A^{\dagger}}} \Phi\left( \frac{\sigma_{a^\dagger}Z_{a^{\dagger}}+\mu_{a^\dagger} - \mu_{a}}{\sigma_{a}}\right)\right)\right]\leq\delta}$ \cmmnt{ \todoeout{In fact, we can not wirte this because the events for $a^{\dagger}\in A^{\dagger}$ are not independent}\todol{yes true :/ }
\begin{align*}
\mathbb{E}_{Z_{a^{\dagger}}}\left( \Pi_{a\neq a^{\dagger}} \Phi\left( \frac{\sigma_{a^\dagger}Z_{a^{\dagger}}+\mu_{a^\dagger} - \mu_{a}}{\sigma_{a}}\right)\right)
 \geq 1 - \delta
\end{align*}}
 where $\Phi$ is the cumulative distribution function of a normally distributed Gaussian random variable. Unfortunately, computing exactly this expectation is an open problem.
 
In the more general case where $|A^{\dagger}|\geq 1$, rewriting the constraints of Problem~\eqref{eq:TS_attack_one_user} is not possible. Following the idea of \cite{liu2019data}, \changebrtwo{for every single target arm $a^\dagger\in A^\dagger$}, a possible relaxation of the constraint in Problem \eqref{eq:TS_attack_one_user} is, \changee{to ensure that there exists an arm $a^{\dagger}\in A^{\dagger}$ such that} for every arm $a\not\in A^\dagger$,    $ 1 - \Phi\left( (\mu_{a^{\dagger}} - \mu_{a} - \xi)/(\sqrt{\sigma_{a}^{2} + \sigma_{a^{\dagger}}^{2}})\right) \leq \frac{\delta}{\changelm{K-|A^\dagger|}}$, \changelm{where $|A
^\dagger|$ is the cardinal of  $A^\dagger$}.
Thus the relaxed version of the attack on \lints \changebrtwo{for a single arm $a^\dagger$} is:
\begin{align}
\min_{y\in \mathbb{R}^{d}} \| y \|  
 \quad \text{s.t.} \quad  \forall \changebrtwo{a\not\in A^{\dagger}},\langle x^{\dagger}+y, \hat{\theta}_{a^{\dagger}} - \hat{\theta}_{a}\rangle - \xi  \geq \nu\Phi^{-1}\left(1 - \tfrac{\delta}{\changelm{K-|A^\dagger|}}\right)\big\| x^{\dagger} + y \big\|_{\bar{V}_{a}^{-1} + \bar{V}_{a^{\dagger}}^{-1}} \label{eq:relaxed_TS_attack_one_user}
\end{align} 
Problem \eqref{eq:relaxed_TS_attack_one_user} is similar to Problem \eqref{eq:relaxed_attack_one_user} as the constraint is also a Second Order Cone Program but with different parameters (see App.~\ref{app:one_context_ts_linucb}). \changebrtwo{As in section \ref{sec:optimistic_algorithms}, we compute the final attack as the minimum of the attacks computed for each arm in $A^\dagger$.}

\vspace{-.1in}
\section{Experiments}\label{sec:experiments}
In this section, we conduct experiments on the attacks on contextual bandit problems with simulated data and two real-word datasets: MovieLens25M \cite{harper2015movielens} and Jester \cite{goldberg2001eigentaste}. The synthetic dataset and the data preprocessing step are presented in App.~\ref{app:experiments_setup}.
\vspace{-0.05in}
\subsection{Attacks on Rewards}
\vspace{-0.02in}
We study the impact of the reward attack for $4$ contextual algorithms: \linucb, \lints, \epsgreedy and \expfour. As parameters, \changebr{we use $L=1$ for the maximal norm of the contexts}, $\delta = 0.01$, $\upsilon = \sigma\sqrt{d\ln(t/\delta))/2}$, $\varepsilon_{t} = 1/\sqrt{t}$ at each time step $t$ and $\lambda = 0.1$. \changee{We choose only a \textit{unique target arm} $a^{\dagger}$.} For \expfour, we use $N = 10$ experts with $N-2$ experts returning a random arm at each time, one expert choosing arm $a^{\dagger}$ every time and one expert returning the optimal arm for every context. With this set of experts the regret of bandits with expert advice is the same as in the contextual case. To test the performance of each algorithm, we generate $40$ random contextual bandit problems and run each algorithm for $T = 10^{6}$ steps on each. We report the average cost and regret for each of the $40$ problems.  
Figure \ref{fig:costs_plot} (Top) shows the attacked algorithms using the attacked reward $\wt{r}^{1}$ (reported as ``stationary CACE'') and the rewards $\wt{r}^{2}$ (reported as CACE).

These experiments show that, even though the reward process is non-stationary,  usual stochastic algorithms like \linucb can still adapt to it and pull the optimal arm for this reward process (which is arm $a^{\dagger}$). The true regret of the attacked algorithms is linear as $a^{\dagger}$ is not optimal for all contexts. 
In the synthetic case, for the algorithms attacked with the rewards $\wt{r}^{2}$, over 1M iterations and $\gamma = 0.22$, the target arm is drawn more than $99.4\%$ of the time on average for every algorithm and more than $97.8\%$ of the time for the stationary attack $\wt{r}^1$ (see Table~\ref{table:number_of_draws} in App.~\ref{app:additional_fig_rwds}). The dataset-based environments (see Figure~\ref{fig:costs_plot} (Left)) exhibit the same behavior: the target arm is pulled more than $94.0\%$ of the time on average for all our attacks on Jester and MovieLens and more than $77.0\%$ of the time in the worst case (for \lints attacked with the stationary rewards) (see Table~\ref{table:number_of_draws}).

\vspace{-0.05in}
\subsection{Attacks on Contexts}\label{subsec:exp_attack_all_context}
\vspace{-0.02in}
\changee{We now illustrate the effectiveness of the attack in Alg.~\ref{alg:context_attack_protocol}.  We study the behavior of attacked \linucb, \lints, \epsgreedy with different size of target arms set ($|A^{\dagger}|/K\in \{ 0.3, 0.6, 0.9\}$ with $K$ the total number of arms).} We test the performance of \linucb with the same parameters as in the previous experiments. Yet since the variance is much smaller in this case, we generate a random problem and run $20$ simulations for each algorithm. \changee{The target arms are chosen randomly and we use the exact lower-bound on the reward of those arms to compute $\nu$.\cmmnt{In addition, to measure the robustness of random exploration algorithm like \lints and \epsgreedy we use an unbounded attack on the context}}

\begin{table}[h]
\begin{center}
\caption{Percentage of iterations for which the algorithm pulled an arm in the target set $A^{\dagger}$ (with a target set size of $0.3K$ arms) (\textbf{Left}) Online attacks using ContextualConic ($CC$) algorithm. Percentages are averaged over 20 runs of 1M iterations. (\textbf{Right}) Offline attacks with exact (Full) and Relaxed optimization problem. Percentages are averaged over 40 runs of 1M iterations. \label{table:success_rate}\vspace{-0.2cm}}
\scalebox{0.8}{
\begin{tabular}{lccc}
\toprule
{} & Synthetic &  Jester & Movilens \\
\midrule
\linucb          &      28.91\% &   26.59\% &    31.13\% \\
CC LinUCB     &     98.55\% &  98.36\% &   99.61\% \\
\epsgreedy      &     25.7\% &   25.85\% &    31.78\% \\
CC \epsgreedy &     89.71\% &  99.85\% &   99.92\% \\
\lints           &      27.2\% &   26.10\% &    33.24\% \\
CC \lints      &      30.93\% &  97.26\% &   98.82\% \\
\bottomrule
\end{tabular}
}
\scalebox{0.7}{
\begin{tabular}{lccc}
\toprule
{} & Synthetic &  Jester & MovieLens \\
\midrule
\linucb             &      $0.07\%$ &   $0.01\%$ &    $0.39\%$ \\
\linucb Relaxed     &     $13.76\%$ &  $97.81\%$ &    $4.09\%$ \\
\linucb Full        &     $88.30\%$ &  $99.98\%$ &   $99.99\%$ \\
\epsgreedy         &      $0.01\%$ &   $0.00\%$ &    $0.03\%$ \\
\epsgreedy Full &     $99.98\%$ &  $99.95\%$ &   $99.97\%$ \\
\lints              &      $0.02\%$ &   $0.01\%$ &    $0.05\%$ \\
\lints Relaxed      &     $18.21\%$ &  $80.48\%$ &    $5.56\%$ \\
\bottomrule
\end{tabular}
}

\end{center}
\end{table}

Table~\ref{table:success_rate} (Left) shows the percentage of times \changee{an arm in $A^{\dagger}$, for $|A^{\dagger}| = 0.3K$}, has been selected by the attacked algorithm\changebr{. We} see that, as expected, CC \linucb reaches a ratio of almost $1$, meaning the target arms are indeed pulled a linear number of times. A more surprising result (at least not covered by the theory) is that \epsgreedy exhibits the same behavior. \changebr{Similarly to \lints, \epsgreedy exhibits some randomness in the action selection process. It can cause an arm $a^{\dagger}\in A^{\dagger}$ to be chosen when the context is attacked and interfere with the principle of the attack.} We suspect that is what happens for \lints. Fig.~\ref{fig:costs_plot} (Bottom) shows the total cost of the attacks for the attacked algorithms \cmmnt{(except for \lints, for which the cost is linear)}. 
\changee{Despite the fact that the estimate of $\theta_{a^\dagger}$ can be polluted by attacked samples, it seems that \lints can still pick up $a^{\dagger}$ as being optimal for some dataset like MovieLens and Jester but not on the simulated dataset.}

\begin{figure}
    \centering
    \begin{minipage}{0.25\linewidth}
        \centering
        \includegraphics[width=0.85\linewidth]{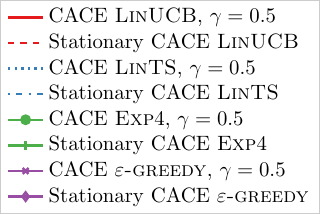}
    \end{minipage}\hfill
    \begin{minipage}{0.25\linewidth}
    \centering
    \includegraphics[width=0.95\linewidth]{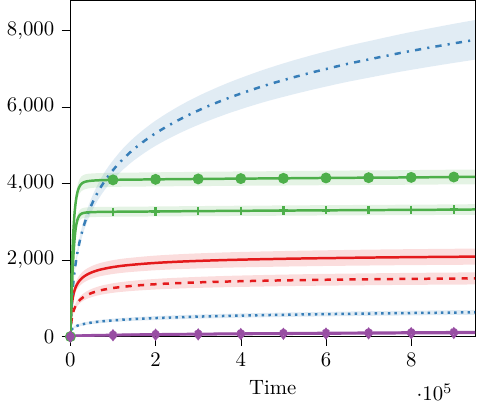}
    \end{minipage}\hfill
    \begin{minipage}{0.25\linewidth}
    \centering
    \includegraphics[width=0.85\linewidth]{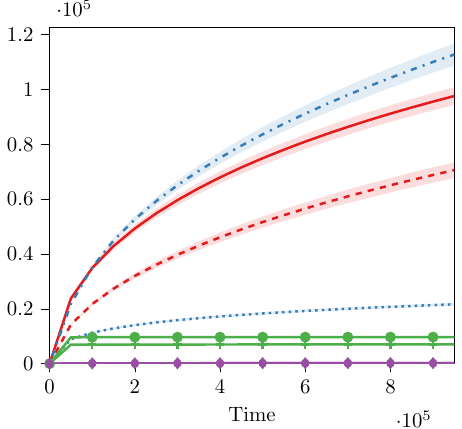}
    \end{minipage}\hfill
    \begin{minipage}{0.25\linewidth}
    \centering
    \includegraphics[width=0.85\linewidth]{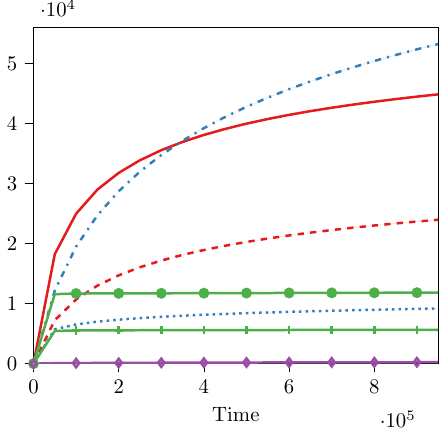}
    \end{minipage}\\
    \begin{minipage}{0.25\linewidth}
        \centering
        \includegraphics[width=0.85\linewidth]{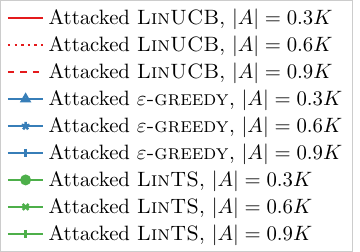}
    \end{minipage}\hfill
    \begin{minipage}{0.25\linewidth}
    \centering
    \includegraphics[width=0.95\linewidth]{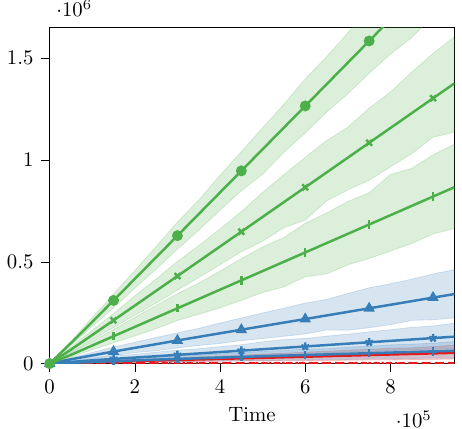}
    \end{minipage}\hfill
    \begin{minipage}{0.25\linewidth}
    \centering
    \includegraphics[width=0.85\linewidth]{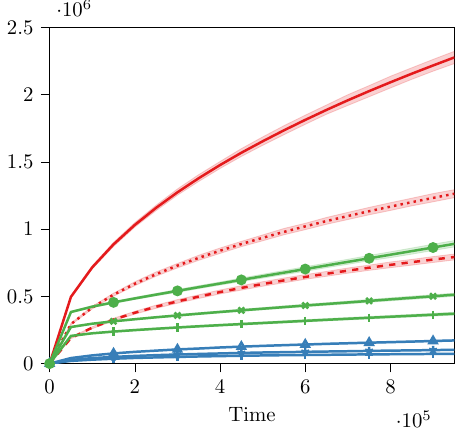}
    \end{minipage}\hfill
    \begin{minipage}{0.25\linewidth}
    \centering
    \includegraphics[width=0.85\linewidth]{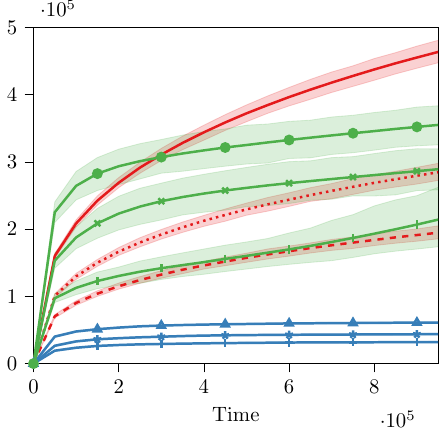}
    \end{minipage}
    \caption{Total cost of attacks on rewards for the synthetic (Left, $\gamma=0.22$), Jester (Center, $\gamma=0.5$) and MovieLens (Right, $\gamma=0.5$) environments. Bottom, total cost of ContextualConic attacks on the synthetic (Left), Jester (Center) and MovieLens (Right) environments.\vspace{-0.5cm}}
\label{fig:costs_plot}
\end{figure}

\vspace{-0.2cm}
\subsection{Offline attacks on a Single Context}
\vspace{-0.02in}
 We now move to the setting described in Sec.~\ref{sec:attack_one_context} and test the same algorithms as in Sec.~\ref{subsec:exp_attack_all_context}. We run 40 simulations for each algorithm \changebr{and each attack type}. The target context $x^{\dagger}$ is chosen randomly and the target arm as the arm minimizing the expected reward for $x^{\dagger}$. 
 The attacker is only able to modify the incoming context for the target context (which corresponds to the context of one user) and the incoming contexts are sampled uniformly from the set of all possible contexts (of size $100$). 
 Table~\ref{table:success_rate} (Right) shows the percentage of success for \changebr{each attack}. We observe that the non-relaxed attacks on \epsgreedy and \linucb work well across all datasets. 
 However, the relaxed attack 
 for \linucb and \lints are not as successful, on the synthetic dataset and MovieLens25M. 
 The Jester dataset seems to be particularly suited to this type of attacks because the true feature vectors are well separated from the convex hull formed by the feature vectors of the other arms: only $5$\% of Jester's feature vectors are within the convex hull of the others versus $8\%$ for MovieLens and $20\%$ for the synthetic dataset.
 As expected, the cost of the attacks is linear on all the datasets (see Figure \ref{fig:cost_attack_one_ctx} in App.~\ref{app:additional_fig_one_ctx}). The cost is also lower for the non-relaxed than for the relaxed version of the attack on \linucb. Unsurprisingly, the cost of the attacks on \lints is the highest 
due to the need to guarantee that 
$a^{\dagger}$ will be chosen with high probability (95\% in our experiments).

\section{Conclusion}
We presented several settings for online attacks on contextual bandits.
We showed that an attacker can force any contextual bandit algorithm to almost always pull an arbitrary target arm $a^{\dagger}$ with only sublinear modifications of the rewards. When the attacker can only modify the contexts, we prove that \linucb can still be attacked and made to almost always pull an arm in $A^{\dagger}$ by adding sublinear perturbations to the contexts. 
When the attacker can only attack a single context, we derive a feasibility condition for the attacks and we introduce a method to compute some attacks of small instantaneous cost for \linucb, \epsgreedy and \lints.
To the best of our knowledge, this paper is the first to describe effective attacks on the contexts of contextual bandit algorithms. Our numerical experiments, conducted on both synthetic and real-world data, validate our results and show that the attacks on all contexts are actually effective on several algorithms and with more permissible settings.

\newpage{}
\section*{Broader Impact}

Adversarial attacks have been a major concerns in the machine learning community for some time \cite{biggio2012poisoning,goodfellow2014explaining, jagielski2018manipulating, li2016data, liu2017robust} as they delve deeply into the robustness of such machine learning systems. Although, adversarial attacks have only been recently studied for bandits and reinforcement learning algorithms \cite{ma2018data,hussenot2019targeted}. Those settings are applied to a wide range of applications such as recommender systems or cooling down data centers \cite{lazic2018data}. 

In adversarial attacks on supervised algorithms and cryptography, it is well-accepted that the study and publication of attack schemes helps build trustful secure systems \cite{anish2018ObfuscatedGradients}. While there is a risk that our methods could be used by malicious attackers, we believe that they will also prompt some practitioners to ensure  such modifications of the rewards or contexts of their data can be detected or even prevented. 


\bibliographystyle{unsrt}
\bibliography{adv}

\newpage{}
\onecolumn
\appendix

\section{Proofs}
In this appendix, we present the proofs of different theoretical results presented in the paper.
\subsection{Proof of Proposition \ref{prop:reward_attack}}\label{app:proof_prop_rewd_attack}

\begin{prop*}
	For any $\delta\in(0, 1/K]$, when using Contextual ACE algorithm (Alg. ~\ref{alg:attacker_rewards}) with perturbed rewards $\tilde{r}^{1}$, with probability at least $1-K\delta$, algorithm $\mathfrak{A}$ pulls \changelm{an arm in $A^{\dagger}$} for $T - o(T)$ time steps and the total cost of attacks is $o(T)$.
\end{prop*}

\begin{proof}
	Let us consider the contextual bandit problem $\mathcal{A}_{1}$, with $K$ arms with contexts $x\in \mathcal{D}$ such that every arm in \changelm{$a^\dagger\in A^\dagger$ has mean reward $\langle \theta_{a^{\dagger}}, x\rangle$} and all other arms has mean $0$. Then the regret of algorithm $\mathfrak{A}$ for this bandit problem is upper-bounded with probability at least $1 - \delta$ by a function $f_{\mathfrak{A}}(T)$ such that $f_{\mathfrak{A}}(T) = o(T)$. In addition, the reward process fed to Alg. $\mathfrak{A}$ by the attacker is a stationary reward process with $\sigma^{2}$-subgaussian noise. Therefore, the number of times algorithm $\mathfrak{A}$ pulls an arm \changelm{not in $A^{\dagger}$ is upper-bounded by \changee{$f_{\mathfrak{A}}(T)/\min_{x\in \mathcal{D}} \Delta(x)$ where for every context $x\in\mathcal{D}$, let $a
^{\dagger}_{\star}(x) := \arg\max_{a\in A^{\dagger}} \langle x, \theta_{a}\rangle$ and $\Delta(x) = \langle x, \theta_{a^{\dagger}_{\star}(x)}\rangle - \max_{a\in A^{\dagger}, a\neq a_{\star}^{\dagger}(x)} \langle x, \theta_{a}\rangle$}}. 

\changelm{In addition, the total cost of the attack is upper-bounded by $\max_{a\in \llbracket 1, K\rrbracket} \max_{x\in \mathcal{D}} |\langle x, \theta_{a}\rangle| (T - N_{A^{\dagger}}(T))$ where $N_{A^{\dagger}}(T)$ is the number of times an arm  in $A^{\dagger}$ has been pulled up to time $T$. Thanks to the previous argument, $T - N_{A^{\dagger}}(T) \leq  f_{\mathfrak{A}}(T)/\min_{x\in \mathcal{D}}\Delta(x)$.}
\end{proof}


\subsection{Proof of Proposition \ref{prop:cost_attack_all_ctx}}\label{app:proof_attack_all_ctx}

\begin{prop*}
Using the attack described in Alg.~\ref{alg:context_attack_protocol}, for any $\delta\in (0, 1/K]$, with probability at least $1 - K\delta$, the number of times \linucb does not pull \changelm{an arm in $A^{\dagger}$} is at most:
\begin{align*}
    \sum_{j\changelm{\notin A^{\dagger}}} N_{j}(T) \leq 32K^{2}\left( \frac{\lambda}{\alpha^{2}} + \sigma^{2}d\log\left(\frac{\lambda d + TL^2\alpha^{2}}{d\lambda\delta}\right) \right)^{3}
\end{align*}
with $N_{j}(T)$ the number of times arm $j$ has been pulled after $T$ steps, $|| \theta_{a}|| \leq S$ for all arms $a$, $\lambda$ the regularization parameter of \linucb and for all $x\in \mathcal{D}$, $||x||_{2}\leq L$. The total cost for the attacker is bounded by:
\begin{align*}
    \sum_{t=1}^{T} c_{t} \leq \frac{64K^{2}}{\nu}\left( \frac{\lambda}{\alpha^{2}} + \sigma^{2}d\log\left(\frac{\lambda d + TL^2\alpha^{2}}{d\lambda\delta}\right) \right)^{3}
\end{align*}
\end{prop*}

\begin{proof}
Let $a_{t}$ be the arm pulled by \linucb at time $t$. For each arms $a$, let $\tilde{\theta}_a(t)$ be the result of the linear regression with the attacked context and $\hat{\theta}_{a}(t, \lambda/\alpha^{2})$ the one with the unattacked context and a regularization of $\frac{\lambda}{\alpha^{2}}$. At any time step $t$, we can write, for all $a\changebrtwo{\not\in A^\dagger}$:

\begin{align*}
\tilde{\theta}_a(t) &=  \left(\lambda I_d + \sum_{l=0, a_{l} = a}^{t} \alpha^{2} x_l x_l^{\intercal}\right)^{-1} \sum_{k=0, a_{k} = a}^{t} r_k \alpha x_{k} \\
&= \frac{1}{\alpha} \left(\frac{\lambda}{\alpha^2} I_d + \sum_{k=0, a_{k} = a}^t x_k x_k^{\intercal}\right)^{-1} \sum_{k=0, a_{k} = a}^t r_k x_k \\
&= \frac{\hat{\theta}_{a}(t,\lambda/\alpha^{2})}{\alpha}
\end{align*}
\changelm{We also note that, since the contexts are not modified for arms in  $a^\dagger\in A^\dagger$: $\tilde{\theta}_{a^\dagger}(t)=\hat{\theta}_{a^\dagger}(t,\lambda)$. In addition, for any context $x$ and arm $a\notin A^\dagger$, the exploration term used by \linucb becomes:}
\begin{align}
    ||x||_{\tilde{V}_{a,t}^{-1}}&= \frac{1}{\alpha} ||x||_{\hat{V}_{a,t}^{-1}}
\end{align}
where $\tilde{V}_{a,t} = \lambda I_d + \sum_{l=0, a_{l} = a}^{t} \alpha^{2} x_l x_l^{\intercal}$ and $\hat{V}_{a,t}^{-1} =\lambda/ \alpha^2 I_d + \sum_{k=0, a_{k} = a}^t x_k x_k^{\intercal}$. For a time $t$, if presented with context $x_{t}$ \linucb pulls arm \changelm{$a_{t} \notin A^{\dagger}$,} we have:
\begin{align*}
\alpha\left(\left\langle \hat{\theta}_{a^\dagger}(t), x_{t} \right\rangle +\beta_{a^\dagger}(t)||x_t||_{V_{a^\dagger,t}^{-1}}\right)\leq \left\langle \hat{\theta}_{a_{t}}(t, \lambda/\alpha^{2}), x_{t} \right\rangle +  \beta_{a_{t}}(t)||x_{t}||_{\hat{V}_{a_{t},t}^{-1}} 
\end{align*}

As \changelm{$\alpha = \frac2\nu\geq\min_{a^\dagger\in A^\dagger}\frac{2}{\left\langle \theta_{a^\dagger}, x_{t} \right\rangle}$}, we deduce that on the event that the confidence sets (Theorem $2$ in \cite{abbasi2011improved}) hold for arm $a^{\star}$: 
\begin{align*}
    2&\leq\left\langle \hat{\theta}_{a_{t}}(t, \lambda/\alpha^{2}), x_{t} \right\rangle +  \beta_{a_{t}}(t)||x_{t}||_{\hat{V}_{a_{t},t}^{-1}}\leq \langle\theta_{a_{t}}, x_{t}\rangle+2\beta_{a_{t}}(t)||x_{t}||_{\hat{V}_{a_{t},t}^{-1}}
\end{align*}
Thus, $1 \leq 2 - \langle\theta_{a_{t}}, x_{t}\rangle \leq 2\beta_{a_{t}}(t)||x_{t}||_{\hat{V}_{a_{t},t}^{-1}}$. Therefore,
\begin{align*}
    \sum_{t=1}^{T} \mathds{1}_{\{a_{t}\notin A^{\dagger}\}} &\leq \sum_{t=1}^{T} \min(2\beta_{a_{t}}(t)||x_{t}||_{\hat{V}_{a_{t},t}^{-1}},1)\mathds{1}_{\{a_{t} \notin A^{\dagger}\}}\\
    &\leq \sum_{j\notin A^{\dagger}} 2\beta_{j}(T)\sqrt{\sum_{t=1}^{T}\mathds{1}_{\{a_{t}=j\}}\sum_{t=1, a_{t}=j}^{T} \min(1, ||x_{t}||^{2}_{\hat{V}_{j,t}^{-1}})}&
  \end{align*}
  But using Lemma $11$ from \cite{abbasi2011improved} and the bound on the $\beta_{j}(T)$ for all arms $j$, we have with Jensen inequality:
  \begin{align*}
    \sum_{t=1}^{T} \mathds{1}_{\{a_{t}\notin A^{\dagger}\}} \leq &4\sqrt{K\sum_{t=1}^{T} \mathds{1}_{\{a_{t}\notin A^{\dagger}\}}d\log\left(1 + \frac{\alpha^2TL^2}{\lambda d}\right)}\\
    &\times\Big( \sqrt{\lambda/\alpha^{2}} S + \sigma\sqrt{2\log(1/\delta) + d\log(1 + \frac{\alpha^2TL^2}{\lambda d})}\Big)
\end{align*}
\end{proof}

\subsection{Proof of Theorem \ref{thm:feasibility_attack_one_user}}\label{app:feasibility_attack_one_user}

\begin{theorem*}
For any $\xi>0$, Problem \eqref{eq:attack_one_user} is feasible if and only if:
\begin{align}\label{eq:feasibilty_condition_bis}
\exists \theta \in  \changelm{\bigcup_{a^\dagger\in A^{\dagger}}}\mathcal{C}_{t, a^{\dagger}}, \qquad \theta\not\in \text{Conv}\left( \bigcup_{a\notin A^{\dagger}} \mathcal{C}_{t,a}\right)
\end{align}
	where for every arm $a$,  $\mathcal{C}_{t,a} := \big\{\theta \mid ||\theta - \hat{\theta}_{a}(t)||_{\tilde{V}_{a,t}} \leq \beta_{a}(t) \big\}$ with $\hat{\theta}_{a}(t)$ the least squares estimate for arm $a$ built by \linucb and 
	$$\tilde{V}_{a,t} = \lambda I_{d} + \sum_{l=1, x_{l}\neq x^{\dagger}}^{t} \mathds{1}_{\{a_{l} = a\}}x_{l}x_{l}^{\intercal} + \sum_{l=1, x_{l}= x^{\dagger}}^{t} \mathds{1}_{\{a_{l} = a\}}\tilde{x}_{l}\tilde{x}_{l}^{\intercal} $$ 
	the design matrix of \linucb at time $t$ for all arms $a$ (where $\tilde{x}_{l}$ is the modified context)
\end{theorem*}

\begin{proof}
The proof of Theorem \ref{thm:feasibility_attack_one_user} is decomposed in two parts. 

First, let us assume that Equation \eqref{eq:feasibilty_condition_bis} is satisfied. Then, \changebrtwo{let us define $a^\dagger \in A^\dagger$ such that} $\theta \in \mathcal{C}_{t,a^{\dagger}}\setminus \text{Conv}\left( \bigcup_{a\notin A^{\dagger}} \mathcal{C}_{t,a}\right) $, then by the theorem of separation of convex sets applied to $\mathcal{C}_{t,a^{\dagger}}$ and $\{ \theta \}$. There exists a vector $v$ and $c_{1}< c_{2}$ such that for all $y \in \text{Conv}\left( \bigcup_{a\neq a^{\dagger}} \mathcal{C}_{t,a}\right)$:
\begin{align*}
\left\langle y, v\right\rangle \leq c_{1} < c_{2} \leq \left\langle \theta,v\right\rangle.
\end{align*}
Hence, for $\xi>0$ we have that for $\tilde{v} = \frac{\xi}{c_{2}-c_{1}} v$ that:
\begin{align*}
    \left\langle y, \tilde{v}\right\rangle + \xi \leq \left\langle \theta, \tilde{v} \right\rangle
\end{align*}
So the problem is feasible.

Secondly, let us assume that an attack is feasible. Then there exists a vector $y$ such that:
\begin{align}
    \changebrtwo{\max_{a^\dagger \in A^\dagger}}\max_{\theta\in \mathcal{C}_{t,a^{\dagger}}} \left\langle y, \theta\right\rangle > c_{1} := \max_{a\notin A^{\dagger}} \max_{\theta\in \mathcal{C}_{t,a}} \left\langle y, \theta\right\rangle
    \label{eq:feasible_in_proof}
\end{align}
\changelm{
	Let us reason by contradiction. We assume that $ \bigcup_{a\in A^{\dagger}}\mathcal{C}_{t,a^{\dagger}} \subset \text{Conv}\left( \bigcup_{a\notin A^{\dagger}} \mathcal{C}_{t,a}\right)$ and consider 
	\begin{align*}
	    \theta^*\in\bigcup_{a\in A^{\dagger}}\mathcal{C}_{t,a^{\dagger}}\text{ such that } \left\langle y, \theta^*\right\rangle=\max_{a^\dagger \in A^\dagger}\max_{\theta\in \mathcal{C}_{t,a^{\dagger}}} \left\langle y, \theta\right\rangle
	\end{align*}
	As we assumed $ \bigcup_{a\in A^{\dagger}}\mathcal{C}_{t,a^{\dagger}} \subset \text{Conv}\left( \bigcup_{a\notin A^{\dagger}} \mathcal{C}_{t,a}\right)$, there exists $n\in\mathbb{N}^{\star}$, $\lambda_{1},\cdots, \lambda_{n}\geq 0$ and $\theta_{1}, \cdots, \theta_{n}\in \bigcup_{a\notin A^{\dagger}} \mathcal{C}_{t,a}$ \text{such that}
	\begin{align*}
	    \theta^* = \sum_{i=1}^{n} \lambda_{i}\theta_{i}\text{ and } \sum_{i=1}^{n} \lambda_{i} = 1
	\end{align*}
	Thus
\begin{align}
    \left\langle y, \theta^*\right\rangle = \sum_{i} \lambda_{i} \left\langle y, \theta_{i} \right\rangle \leq c_{1}\sum_{i=1}^{n} \lambda_{i} = c_{1}\label{cdas}
\end{align}
\changebrtwo{We assumed that the problem is feasible, so $c_{1}<  
\left\langle y, \theta^*\right\rangle$ according to Eq.~\ref{eq:feasible_in_proof}. It} contradicts Eq. \ref{cdas}.
}
\end{proof}

\subsection{Condition of Sec.~\ref{sec:attack_one_context}}\label{app:condition_linear}
\begin{figure}[h]\label{fig:feasibility_condition}
\centering
    \includegraphics[width=0.5\linewidth]{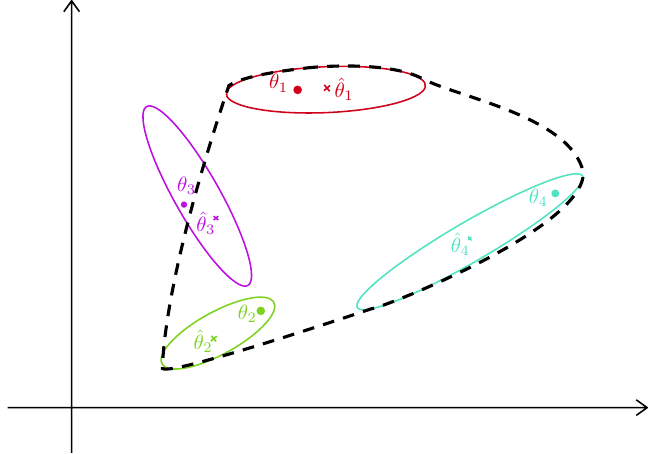}
	\caption{Illustrative example of condition \eqref{eq:feasibilty_condition}. The target arm is arm $3$ or $5$ and the dashed black line is the convex hull of the other confidence sets. The ellipsoids are the confidence sets $\mathcal{C}_{t,a}$ for each arm $a$. If we consider only arms $\{1,2,4,5\}$, and we use $5$ as the target arm, the condition \eqref{eq:feasibilty_condition} is satisfied as there is a $\theta$ outside the convex hull of the other confidence sets. On the other hand, if we consider arms $\{1,2,3,4\}$ and we use $3$ as the target arm, the condition is not satisfied anymore.}
\vspace{-.1in}
\end{figure}

\changebr{Let us assume that \changebrtwo{there is an arm in $a^\dagger\in A^\dagger$ which is} optimal for some contexts. More formally, there exists a subspace $V\subset \mathcal{D}$ such that:} 
\begin{equation*}    
\forall x\in V, \exists a^{\dagger}_{\star}(x)\in A^\dagger, \forall a\in \llbracket 1, K\rrbracket\setminus\{a^{\dagger}_{\star}(x)\} \qquad \langle x, \theta_{a^{\dagger}_{\star}(x)}\rangle > \left\langle x, \theta_{a}\right\rangle.
\end{equation*}
\changebr{We also assume that} the distribution of the contexts is such that, for all $t$, $\mu := \mathbb{P}\left(x_{t}\in V\right) >0$.
Then\changebr{,} the regret is lower-bounded in expectation by:
\begin{align*}
    \mathbb{E}(R_{T}) &= \mathbb{E}\left(\sum_{t=1}^{T} \mathds{1}_{\{x_{t}\in V\}}\big( \left\langle x_{t}, \theta_{a^{\dagger}_{\star}(x_{t})} - \theta_{a_{t}}\right\rangle\big)\right) \geq \mu m(T) \min_{x\in V} \max_{a\neq a^{\dagger}_\star(x)} \langle \theta_{a^{\dagger}_{\star}(x)} - \theta_{a}, x\rangle
\end{align*}
where $m(T)$ is the expected number of times $t\leq T$ such that condition \eqref{eq:feasibilty_condition} is not met. \changebr{\linucb guarantees that}  $\mathbb{E}(R_{T}) \leq \mathcal{O}(\sqrt{T})$ for every $T$. Hence, $m(T) \leq \mathcal{O}\left(\frac{\sqrt{T}}{\mu\min_{x\in V}\max_{a\neq a^{\dagger}_\star(x)} \langle \theta_{a^{\dagger}_{\star}(x)} - \theta_{a}, x\rangle}\right)$. This means that, in an unattacked problem, condition \eqref{eq:feasibilty_condition} is met $T - \mathcal{O}(\sqrt{T})$ times. On the other hand, when the algorithm is attacked the regret of \linucb is not sub-linear as the confidence bound for the target arm is not valid anymore. Hence we cannot provide the same type of guarantees for the attacked problem.

\section{Experiments}

\subsection{Datasets and preprocessing}\label{app:experiments_setup}

We present here the datasets used in the article and how we preprocess them for numerical experiments conducted in Section \ref{sec:experiments}.

We consider two types of experiments, one on synthetic data with a contextual MAB problems with $K = 10$ arms such that for every arm $a$, $\theta_{a}$ is drawn from a folded normal distribution in dimension $d = 30$. We also use a finite number of contexts ($10$), each of them is drawn from a folded normal distribution projected on the unit circle multiplied by a uniform radius variable (i.i.d. across all contexts). Finally, we scale the expected rewards in $(0,1]$ and the noise is drawn from a centered Gaussian distribution $\mathcal{N}(0, 0.01)$. 

The second type of experiments is conducted in the real-world datasets Jester \cite{goldberg2001eigentaste} and MovieLens25M \cite{harper2015movielens}. Jester consists of joke ratings on a continuous scale from $-10$ to $10$ for $100$ jokes from a total of $73421$ users. We use the features extracted via a low-rank matrix factorization ($d = 35$) to represent the actions (i.e., the jokes). We consider a complete subset of $40$ jokes and $19181$ users . Each user  rates all the $40$ jokes. At each time, a user is randomly selected from the $19181$ users and mean rewards are normalized in $[0, 1]$. The reward noise is $\mathcal{N} (0, 0.01)$. The second dataset we use is MovieLens25M. It contains $25000095$ ratings created by $162541$ users on $62423$ movies. We perform a low-rank matrix factorization to compute users features and movies features. We keep only movies with at least $1000$ ratings, which leave us with $162539$ users and $3794$ movies. At each time step, we present a random user, and the reward is the scalar product between the user feature and the recommend movie feature. All rewards are scaled to lie in $[0,1]$ and a Gaussian noise $\mathcal{N}(0, 0.01)$ is added to the rewards.

\subsection{Attacks on Rewards}\label{app:additional_fig_rwds}
In this appendix, we present empirical evolution of the total cost and the number of draws for a unique target arm as a function of the attack parameter $\gamma$ for the Contextual ACE attack with perturbed rewards $\tilde{r}^{2}$ on generated data.

\begin{figure}[htbp]
    \centering
    \subfigure[Total cost]{\includegraphics[width=0.35\textwidth]{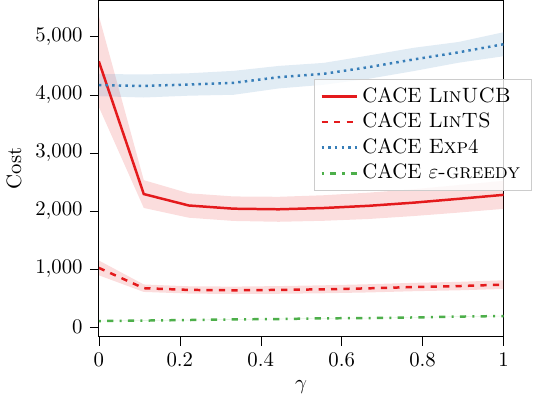}}
    \subfigure[Number of draws]{\includegraphics[width=0.3\textwidth]{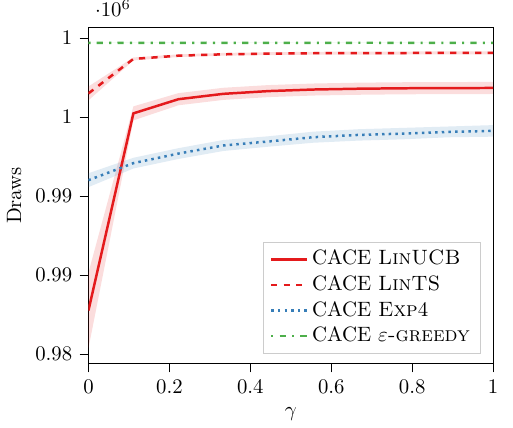}}
    \caption{Total cost of attacks and number of draws of the target arm at $T = 10^{6}$ as a function of $\gamma$ on synthetic data}
    \label{fig:synth_cost_draws_gamma}
\end{figure}

Fig.~\ref{fig:synth_cost_draws_gamma} (left) shows that the total cost of attacks seems to be quite invariant w.r.t.  $\gamma$ except when $\gamma \rightarrow 0$ because the difference between the target arm and the other becomes negligible. This is also depicted by the total number of draws (Fig.~\ref{fig:synth_cost_draws_gamma}, Right) as the number of draws plummets when $\gamma \rightarrow 0$.

\begin{table}
\begin{center}
	\caption{\label{table:number_of_draws}Number of draws of the target arm $a^{\dagger}$ at $T=10^{6}$, for the synthetic data, $\gamma = 0.22$ for the Contextual ACE algorithm and for the Jester and MovieLens datasets $\gamma = 0.5$\otc{.}}
\begin{tabular}{lccc}
\toprule
{} & Synthetic &  Jester & Movilens \\
\midrule
\linucb          &      $86, 731.6$ &  $23, 548.16$ &    $25, 017.31$ \\
CACE \linucb     &     $996, 238.6$ &  $921, 083.69$ &   $944, 721.28$ \\
Stationary CACE \linucb &     $995, 578.88$ & $862, 095.67$ &   $931, 531.6$ \\
\epsgreedy       &     $111, 380.44$ & $21, 911.54$ &    $3, 165.81$    \\
CACE \epsgreedy  &    $999, 812.92$ &  $999, 755.72$ &   $999, 776.82$ \\
Stationary CACE \epsgreedy &     $999, 806.32$ &  $999, 615.98$ &   $999, 316.76$ \\
\lints           &      $91, 664.8$ &  $23, 398.3$ &    $30, 189.84$ \\
CACE \lints      &      $998, 997.04$ &   $976, 708.9$ &   $990, 250.67$ \\
Stationary CACE \lints &     $977, 850.96$ & $784, 715.62$ &   $845, 512.98$ \\
\expfour         &     $93, 860.4$ &  $29, 147.01$ &    $17, 985.78$ \\
CACE \expfour    &    $992, 793.36$ &   $989, 214.36$ &    $936, 230.4$ \\
Stationary CACE \expfour &     $993, 673.24$ &  $988, 463.56$ &   $934, 304.23$ \\
\bottomrule
\end{tabular}

\end{center}
\end{table}

\subsection{Attacks on all Contexts}\label{app:additional_fig_all_ctx}



\begin{figure}[h]
   \begin{minipage}{0.25\linewidth}
        \centering
        \includegraphics[width=0.85\linewidth]{images/regret_cost_attacks_context/simulations/legend.pdf}
    \end{minipage}\hfill
    \begin{minipage}{0.25\linewidth}
    \centering
    \includegraphics[width=0.95\linewidth]{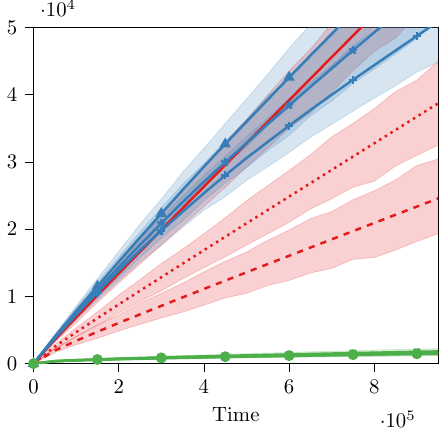}
    \end{minipage}\hfill
    \begin{minipage}{0.25\linewidth}
    \centering
    \includegraphics[width=0.85\linewidth]{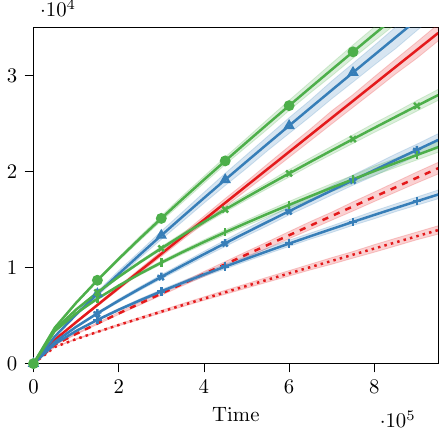}
    \end{minipage}\hfill
    \begin{minipage}{0.25\linewidth}
    \centering
    \includegraphics[width=0.85\linewidth]{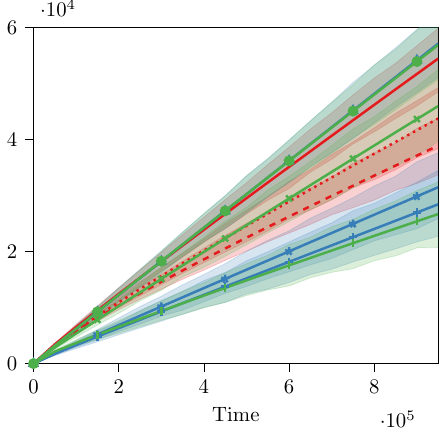}
    \end{minipage}
    \label{fig:regret_all_algs_attack_all_ctx}
\end{figure}

Fig.~\ref{fig:regret_all_algs_attack_all_ctx} shows the regret for all the attacks. This figure shows that even though the total cost of attacks is linear for algorithms like \lints in the synthetic dataset, the regret is linear. More generally, we observe that the regret is linear for all attacked algorithms on all datasets.

\subsection{Attack on a single context}\label{app:additional_fig_one_ctx}

The attacks are computed by solving the optimization problems \ref{eq:attack_one_user} and \ref{eq:relaxed_attack_one_user} (Sec.~\ref{sec:attack_one_context}). We choose the libraries according to their efficiency for each problem we need to solve. For Problem \eqref{eq:relaxed_attack_one_user} and Problem \eqref{eq:relaxed_TS_attack_one_user} we use \textsc{cvxpy}  \cite{cvxpy_rewriting} and the \textsc{ECOS} solver. 
For Problem \eqref{eq:attack_one_user} we use the \textsc{SLSQP} method from the Scipy optimize library \cite{scipy} to solve the full \linucb problem (Equation \ref{eq:attack_one_user}) and \textsc{quadprog} to solve the quadratic problem to attack \epsgreedy.

\begin{figure}[htbp]
    \centering
    \subfigure[Synthetic data]{\includegraphics[width=0.33\textwidth]{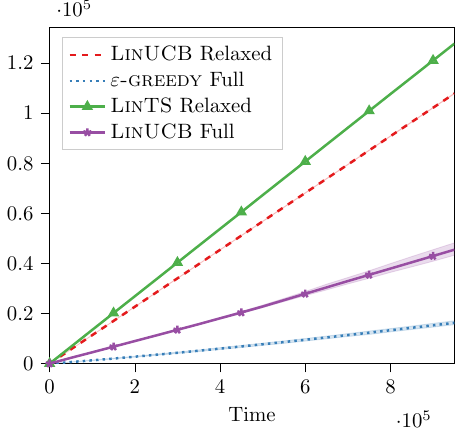}}\hfill
    \subfigure[Jester Dataset]{\includegraphics[width=0.33\textwidth]{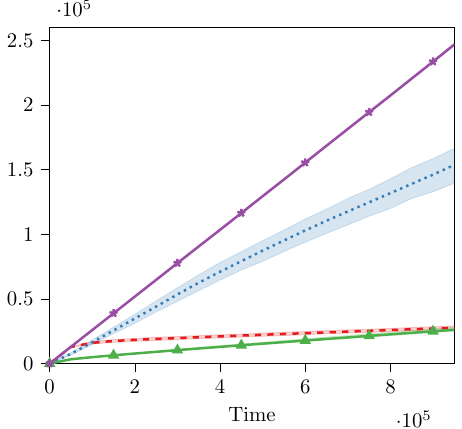}}\hfill
    \subfigure[MovieLens Dataset]{\includegraphics[width=0.33\textwidth]{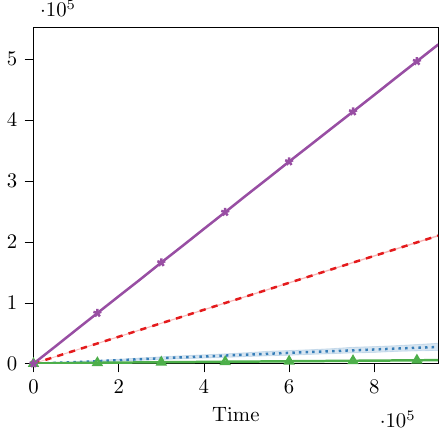}}
    \caption{Total cost of the attacks for the attacks one one context on our synthetic dataset, Jester and MovieLens. As expected, the total cost is linear.}
    \label{fig:cost_attack_one_ctx}
\end{figure}
\section{Problem \eqref{eq:relaxed_TS_attack_one_user} as a Second Order Cone (SOC) Program}\label{app:one_context_ts_linucb}
Problem \eqref{eq:relaxed_attack_one_user} and Problem \eqref{eq:relaxed_TS_attack_one_user} are both SOC programs. We can see the similarities between both problems as follows. Let us define for every arm \changee{$a\not\in A^{\dagger}$}, the ellipsoid:

$$\mathcal{C}_{t,a}^{'} := \Big\{ y \in \mathbb{R}^{d} \mid || y - \hat{\theta}_{a}(t)||_{A_{a}^{-1}(t)} \leq \upsilon\Phi^{-1}\left( 1 - \frac{\delta}{K-|A^{\dagger}|}\right)\Big\}$$

with $A_{a}(t) = \tilde{V}_{a}^{-1}(t) + \tilde{V}_{a^{\dagger}}^{-1}(t)$ with $\tilde{V}_{a}(t)$ and $\tilde{V}_{a^{\dagger}}(t)$ the design matrix built by \lints and $\hat{\theta}_{a}(t)$ the least squares estimate of $\theta_a$ at time $t$. Therefore for an arm $a$, the constraint in Problem \eqref{eq:relaxed_TS_attack_one_user} can be written for any $y\in\mathbb{R}^{d}$ and some arm $a^{\dagger}\in A^{\dagger}$ as:
\begin{align*}
    \left\langle x^{\star}+y, \hat{\theta}_{a^{\dagger}}(t)\right\rangle - \xi \geq \max_{z\in \mathcal{C}_{t,a}^{'}} \left\langle z, x^{\star} + y\right\rangle
\end{align*}
Indeed for any $x\in \mathbb{R}^{d}$,
\begin{align*}
    \max_{y\in \mathcal{C}_{t,a}^{'}} \left\langle y,x\right\rangle &= \left\langle x, \hat{\theta}_{a}(t)\right\rangle + \upsilon\Phi^{-1}\left( 1 - \frac{\delta}{K-|A^{\dagger}|}\right)\times\max_{ ||A_{a}^{-1/2}(t)u||_{2}\leq 1} \left\langle u, x\right\rangle\\
    &= \left\langle x, \hat{\theta}_{a}(t)\right\rangle + \upsilon\Phi^{-1}\left( 1 - \frac{\delta}{K-|A^{\dagger}|}\right)\max_{ ||z||_{2}\leq 1} \left\langle z, A_{a}^{1/2}(t)x\right\rangle \\
    &= \left\langle x, \hat{\theta}_{a}(t)\right\rangle + \upsilon\Phi^{-1}\left( 1 - \frac{\delta}{K-|A^{\dagger}|}\right)\lVert A_{a}^{1/2}(t)x\rVert_{2}
\end{align*}
Thus, the constraint is feasible if and only if:
\begin{align*}
    \hat{\theta}_{a^{\dagger}}(t) \not\in \text{Conv}\left( \bigcup_{a\not\in A^{\dagger}}  \mathcal{C}_{t,a}^{'}\right)
\end{align*}

\section{Attacks on Adversarial Bandits}\label{app:adversarial_rewards}
In the previous sections, we studied algorithms with sublinear regret $R_T$, \ie mainly bandit algorithms designed for stochastic stationary environments. Adversarial algorithms like \expfour do not provably \otc{enjoy} a \changee{sublinear \textbf{stochastic} regret $R_{T}$ (as defined in the introduction) \footnote{\expfour enjoys a sublinear hindsight regret though. Showing a sublinear upper-bound for the stochastic regret of \expfour is still an open problem (see Section $29.1$ in \cite{lattimore2018bandit})}}. In addition, because this type of algorithms are, by design, robust to non-stationary environments, one could expect them to induce a linear cost on the attacker. In this section, we show that this is not the case for most contextual adversarial algorithms. Contextual adversarial algorithms are studied through the reduction to the bandit with expert advice problem. This is a bandit problem with $K$ arms where at every step, $N$ experts suggest a probability distribution over the arms. The goal of the algorithm is to learn which expert gets the best expected reward in hindsight after $T$ steps. The regret in this type of problem is defined as $R_{T}^{\text{exp}} = \mathbb{E}\left( \max_{m\in \llbracket 1, N \rrbracket}\sum_{t=1}^{T} \sum_{j=1}^{K} E_{m,j}^{(t)}r_{t,j} - r_{t,a_{t}}\right)$
where $E_{m,j}^{(t)}$ is the probability of selecting arm $j$ for expert $m$. In the case of contextual adversarial bandit\changebr{s}, the experts first observe the context $x_{t}$ before recommending an expert $m$. 
Assuming the current setting with linear rewards, we can show that if an algorithm $\mathfrak{A}$, like \expfour, enjoys a sublinear regret $R_{T}^{\text{exp}}$, then, using the Contextual ACE attack with either $\tilde{r}^{1}$ or $\tilde{r}^{2}$, the attacker can fool the algorithm into pulling arm $a^{\dagger}$ a linear number of times under some \otc{mild} assumptions. However, attacking contexts for this type of algorithm is difficult because, even though the rewards are linear, the experts are not assumed to use a specific model for selecting an action.


\begin{prop}\label{prop:rwd_attack_adv}
	Suppose an adversarial algorithm $\mathfrak{A}$ satisfies a regret $R_{T}^{\exp}$ of order $o(T)$ for any bandit problem and that there exists an expert $m^{\star}$ such that $ T - \sum_{t=1}^{T} \mathbb{E}\left(E^{(t)}_{m^{\star}, a_{t,\star}^{\dagger}}\right) = o(T)$ with $a_{t,\star}^{\dagger}$ the optimal arim in $A^{\dagger}$ at time $t$. Then attacking alg. $\mathfrak{A}$ with Contextual ACE leads to pulling arm $a^{\dagger}$, $T-o(T)$ of times in expectation with a total cost of $o(T)$ for the attacker.
\end{prop}

\begin{proof}
Similarly to the proof of Proposition \ref{prop:reward_attack}, let's define the bandit with expert advice problem, $\mathcal{A}_{i}$, such that at each time $t$ the reward vector is $(\tilde{r}^{i}_{t,a})_{a}$ (with $i\in\{1, 2\}$). The regret of this algorithm is: $\Tilde{R}_{T}^{i,\text{exp}} = \mathbb{E}\left( \max_{m\in \llbracket 1, N \rrbracket}\sum_{t=1}^{T} E_{m}^{(t)}\Tilde{r}^i_{t} - \Tilde{r}^i_{t,a_{t}}\right)\in o(T)$. The regret of the learner is: $\mathbb{E}\left( \max_{m\in \llbracket 1, N \rrbracket}\sum_{t=1}^{T} E_{m}^{(t)}r_{t} - r_{t,a_{t}}\right)$ where $a_t$ are the actions taken by algorithm $\mathcal{A}_i$ to minimize $\Tilde{R}_{T}^{i,\text{exp}}$. Then we have:
\begin{align*}
    \Tilde{R}_{T}^{i,\text{exp}} \geq \mathbb{E}\left(\sum_{t=1}^{T}\sum_{j=1}^{K} (E_{m^{\star}, j}^{(t)} - \mathds{1}_{\{\changee{j = a_{t,\star}^{\dagger}}\}})\tilde{r}_{t,j}^{i} + \sum_{t=1}^{T} \tilde{r}^{i}_{t, a^{\dagger}_{t,\star}} - \tilde{r}^{i}_{t,a_{t}} \right)
\end{align*}
Therefore, 
\begin{align*}
  \mathbb{E}\left(\sum_{t=1}^{T} \tilde{r}^{i}_{t, a^{\dagger}_{t, \star}} - \tilde{r}^{i}_{t,a_{t}} \right) &\leq \Tilde{R}_{T}^{i,\text{exp}} + \mathbb{E}\left(\sum_{t=1}^{T}\sum_{j=1}^{K} (\mathds{1}_{\{\changee{j = a_{t,\star}^{\dagger}}\}} - E_{m^{\star}, j}^{(t)})\tilde{r}_{t,j}^{i}\right) \\
  &\leq \Tilde{R}_{T}^{i,\text{exp}} + \mathbb{E}\left(\sum_{t=1}^{T}(1 - E_{m^{\star}, a^{\dagger}_{t,\star}}^{(t)})\tilde{r}_{t,j}^{i}\right) \\
  &\leq \Tilde{R}_{T}^{i,\text{exp}} + \mathbb{E}\left(\sum_{t=1}^{T}(1 - E_{m^{\star}, a^{\dagger}_{t,\star}}^{(t)})\right)
\end{align*}

For strategy $i=1$, we have:
\begin{align*}
    \mathbb{E}\left(\sum_{t=1}^{T} \tilde{r}^{1}_{t, a^{\dagger}_{t, \star}} - \tilde{r}^{1}_{t,a_{t}} \right) &=
    \changee{\sum_{t=1}^{T} \mathbb{E}\left(r_{t,a_{t,\star}^{\dagger}} - \mathds{1}_{\{ a_{t}\in A^{\dagger}\}}\right)\geq \left(T-\mathbb{E}\left(\sum_{t=1}^{T} \mathds{1}_{\{ a_{t} = a_{t,\star}^{\dagger}\}}\right)\right)\Delta}
\end{align*}
\changee{where $\Delta := \min_{x\in \mathcal{D}}\left\{ \langle \theta_{a^\dagger(x)}, x\rangle - \max_{a\in A^{\dagger},a\neq a^{\dagger}(x)} \langle \theta_{a'}, x\rangle\right\}$ with $a^{\dagger}(x) := \arg\max_{a\in A^{\dagger}} \langle \theta_{a}, x\rangle$}. Then, as $\Tilde{R}_{T}^{1,\text{exp}}\in o(T)$ and $\mathbb{E}\left(\sum_{t=1}^{T}(1 - E_{m^{\star}, a^{\dagger}_{t, \star}}^{(t)})\right)\in o(T)$, we deduce that $\mathbb{E}(\sum_{t} \mathds{1}_{\{ a_{t} = a_{t,\star}
^{\dagger}\}}) = T-o(T)$. 

For strategy $i=2$, and $\delta>0$, let us denote by $E_{\delta}$ the event that all confidence intervals hold with probability $1 - \delta$. But on the event $E_{\delta}$, for a time $t$ where $a_{t}\neq a^{\dagger}_{t,\star}$ and such that $-1\leq C_{t,a_{t}} \leq 0$:
\begin{align*}
\tilde{r}^{2}_{t,a_{t}} = r_{t, a_{t}} + C_{t,a_{t}} &= (1 - \gamma)\min_{a^{\dagger}\in A^{\dagger}} \min_{\theta\in \mathcal{C}_{t,a^{\dagger}}} \langle \theta, x_{t}\rangle + \eta_{a_{t},t} + \langle\theta_{a}, x_{t}\rangle - \max_{\theta\in \mathcal{C}_{t,a_{t}}} \langle \theta, x_{t}\rangle \\
&\leq (1 - \gamma) \langle \theta_{a^{\dagger}_{t,\star}}, x_{t}\rangle + \eta_{a_{t},t}
\end{align*}
when $C_{t,a_{t}} >0$ (still on the event $E_{\delta}$):
\begin{align*}
\tilde{r}^{2}_{t,a_{t}} = r_{t,a_{t}} \leq (1 - \gamma) \langle \theta_{a^{\dagger}_{t,\star}}, x_{t}\rangle + \eta_{a_{t},t}
\end{align*}
because $C_{t,a_{t}}>0$ means that $(1 - \gamma) \langle \theta_{a^{\dagger}_{t,\star}}, x_{t}\rangle \geq (1 - \gamma)\min_{a^{\dagger}\in A^{\dagger}}\min_{\theta\in \mathcal{C}_{t,a^{\dagger}}} \langle \theta, x_{t}\rangle \geq \max_{\theta\in \mathcal{C}_{t,a_{t}}} \langle \theta, x_{t}\rangle \geq \langle \theta_{a}, x_{t}\rangle$. But finally, when $C_{t,a_{t}} \leq -1$, $\tilde{r}^{2}_{t,a_{t}} = r_{t,a_{t}} -1 \leq \eta_{a_{t},t} \leq (1- \gamma)\langle \theta_{a^{\dagger}_{t,\star}}, x_{t}\rangle + \eta_{a_{t},t}$.  But on the complementary event $E_{\delta}^{c}$,  $ \tilde{r}^{2}_{t,a_{t}} \leq r_{t,a_t}$. Thus, given that the expected reward is assumed to be bounded in $(0,1]$ (Assumption~\ref{assumption1}):
\begin{align*}
    \mathbb{E}\left(\sum_{t=1}^{T} \tilde{r}^{2}_{t, a^{\dagger}_{t,\star}} - \tilde{r}^{2}_{t,a_{t}} \right) & =  \mathbb{E}\left(\sum_{t=1}^{T} (r_{t, a^{\dagger}} - \tilde{r}^{2}_{t,a_{t}})\mathds{1}_{\{a_{t}\neq a^{\dagger}_{t,\star}\}} \right) \\
    &\geq \mathbb{E}\left(\sum_{t=1}^{T} \min\{\gamma\min_{x\in\mathcal{D}} \langle x, \theta_{a^{\dagger}_{t,\star}}\rangle, \Delta\} \mathds{1}_{\{a_{t}\neq a^{\dagger}_{t,\star}\}}\mathds{1}_{\{E_{\delta}\}}\right)-T\delta
\end{align*}
Finally, putting everything together we have:
\begin{align*}
    \mathbb{E}&\left(\sum_{t=1}^{T} \gamma\min_{x\in\mathcal{D}} \langle x, \theta_{a^{\dagger}_{t,\star}}\rangle \mathds{1}_{\{a_{t}\neq a^{\dagger}_{t,\star}\}}\right) \leq \Tilde{R}_{T}^{2,\text{exp}} + \mathbb{E}\left(\sum_{t=1}^{T}(1 - E_{m^{\star}, a^{\dagger}_{t,\star}}^{(t)})\right) + \delta T \left(\min\{\gamma\min_{a^{\dagger}\in A^{\dagger}}\min_{x\in\mathcal{D}} \langle x, \theta_{a^{\dagger}}\rangle, \Delta\} +1\right)
\end{align*}
Hence, because $\Tilde{R}_{T}^{1,\text{exp}} = o(T)$ and $\mathbb{E}\left(\sum_{t=1}^{T}(1 - E_{m^{\star}, a^{\dagger}}^{(t)})\right) = o(T)$ we have that for $\delta \leq 1/T$, the expected number of pulls of the optimal arm in $A^{\dagger}$ is of order $o(T)$. In addition, the cost for the attacker is bounded by: 
\begin{align*}
\mathbb{E}\left(\sum_{t=1}^{T} c_{t}\right) &= \mathbb{E}\left(\sum_{t=1}^{T} \mathds{1}_{\{a_{t}\neq a^{\dagger}_{t,\star}\}} \big|\max(-1, \min(C_{t, a_{t}},0))\big| \right)\leq  \mathbb{E}\left( \sum_{t=1}^{T} \mathds{1}_{\{a_{t}\neq a^{\dagger}_{t,\star}\}}\right)
\end{align*}
\end{proof}

The proof is similar to the one of Prop.~\ref{prop:reward_attack}. The condition on the expert in Prop.~\ref{prop:rwd_attack_adv} means that there exists an expert which believes an arm $a^{\dagger}\in A^{\dagger}$ is optimal most of the time. The adversarial algorithm will then learn that this expert is optimal. 
Algorithm \expfour has a regret $R_{T}^{\text{exp}}$ bounded by $\sqrt{2TK\log(N)}$, thus the total number of pulls of arms not in $A^{\dagger}$ 
is bounded by $\sqrt{2TK\log(M)}/\gamma$. This result also implies that for adversarial algorithms like \expthree \cite{auer2002finite}, the same type of attacks could be used to fool $\mathfrak{A}$ into pulling arms in $A^{\dagger}$ because the MAB problem can be seen as a reduction of the contextual bandit problem with a unique context and one expert for each arm.


\section{Contextual Bandit Algorithms}\label{app:algorithms}

In this appendix, we present the different bandit algorithms studied in this paper. All algorithms we consider except \expfour uses disjoint models for building estimate of the arm feature vectors $(\theta_{a})_{a\in\llbracket 1, K\rrbracket}$. Each algorithm (except \expfour) builds least squares estimates of the arm features.

\begin{algorithm}[h]
  \caption{Contextual \linucb}
  \label{alg:linucb}
\begin{algorithmic}
  \STATE {\bfseries Input:} regularization  $\lambda$, number of arms $K$, number of rounds $T$, bound on context norms: $L$, bound on norms $\theta_{a}$: $D$
  \STATE Initialize for every arm $a$, $\bar{V}_{a}^{-1}(t) = \frac1\lambda I_{d}$, $\hat{\theta}_{a}(t) = 0$ and $b_{a}(t) = 0$
  \FOR{$t=1,..., T$}
  \STATE Observe context $x_{t}$
  \STATE Compute $\beta_{a}(t) = \sigma\sqrt{d\log\left(\frac{1 +  N_{a}(t)L^{2}/\lambda}{\delta}\right)}$ with $N_{a}(t)$ the number of pulls of arm $a$
  \STATE Pull arm  $a_{t} =\argmax_a \langle \hat{\theta}_{a}(t),x_t\rangle + \beta_{a}(t)||x_{t}||_{\bar{V}_{a}^{-1}(t)}$
  \STATE Observe reward $r_{t}$ and update parameters $\hat{\theta}_{a}(t)$ and $\bar{V}_{a}^{-1}(t)$ such that:
  \begin{align*}
      \bar{V}_{a_{t}}(t+1) = \bar{V}_{a_{t}}(t) + x_{t}x_{t}^{\intercal},\quad b_{a_{t}}(t+1) = b_{a_{t}}(t) + r_{t}x_{t},\quad\theta_{a_{t}}(t+1) = \bar{V}_{a_{t}}^{-1}(t+1)b_{a_{t}}(t+1)
  \end{align*}
  \ENDFOR
\end{algorithmic}
\end{algorithm}

\begin{algorithm}[h]
  \caption{Linear Thompson Sampling with Gaussian prior}
  \label{alg:linTS}
\begin{algorithmic}
  \STATE {\bfseries Input:} regularization  $\lambda$, number of arms $K$, number of rounds $T$, variance $\upsilon$
  \STATE Initialize for every arm $a$, $\bar{V}_{a}^{-1}(t) = \lambda I_{d}$ and $\hat{\theta}_{a}(t) = 0$, $b_{a}(t) = 0$
  \FOR{$t=1,..., T$}
  \STATE Observe context $x_{t}$
  \STATE Draw $\tilde{\theta}_{a}\sim\mathcal{N}(\hat{\theta}_{a}(t), \upsilon^{2}\bar{V}_{a}^{-1}(t))$
  \STATE Pull arm $a_{t} = \argmax_{a\in \llbracket 1, K\rrbracket} \left\langle \tilde{\theta}_{a}, x_{t}\right\rangle$
  \STATE Observe reward $r_{t}$ and update parameters $\hat{\theta}_{a}(t)$ and $\bar{V}_{a}^{-1}(t)$
    \begin{align*}
      \bar{V}_{a_{t}}(t+1) = \bar{V}_{a_{t}}(t) + x_{t}x_{t}^{\intercal},\quad b_{a_{t}}(t+1) = b_{a_{t}}(t) + r_{t}x_{t},\quad\theta_{a_{t}}(t+1) = \bar{V}_{a_{t}}^{-1}(t+1)b_{a_{t}}(t+1)
  \end{align*}
  \ENDFOR
\end{algorithmic}
\end{algorithm}

\begin{algorithm}[h]
  \caption{\epsgreedy}
  \label{alg:egreedy}
\begin{algorithmic}
  \STATE {\bfseries Input:} regularization  $\lambda$, number of arms $K$, number of rounds $T$, exploration parameter $(\varepsilon)_{t}$
	\STATE Initialize, for all arms $a$, $\bar{V}_{a}^{-1}(t) = \lambda I_{d}$ and $\hat{\theta}_{a}(t) = 0$, $\varepsilon_{t} = 1$, $b_{a}(t) = 0$
  \FOR{$t=1,..., T$}
  \STATE Observe context $x_{t}$
  \STATE With probability $\varepsilon_{t}$, pull $a_{t} \sim \mathcal{U}\left(\llbracket 1,K\rrbracket\right)$, or pull $a_{t} = \argmax \langle \theta_{a}, x_{t}\rangle$ 
  \STATE Observe reward $r_{t}$ and update parameters $\hat{\theta}_{a}(t)$ and $\bar{V}_{a}^{-1}(t)$
    \begin{align*}
      &\bar{V}_{a_{t}}(t+1) = \bar{V}_{a_{t}}(t) + x_{t}x_{t}^{\intercal},\quad b_{a_{t}}(t+1) = b_{a_{t}}(t) + r_{t}x_{t},\\
      &\theta_{a_{t}}(t+1) = \bar{V}_{a_{t}}^{-1}(t+1)b_{a_{t}}(t+1)
  \end{align*}
  \ENDFOR
\end{algorithmic}
\end{algorithm}

\begin{algorithm}[h]
  \caption{\expfour}
  \label{alg:exp4}
\begin{algorithmic}
	\STATE {\bfseries Input:} number of arms $K$, experts: $(E_{m})_{m\in\llbracket 1, N\rrbracket}$, parameter $\eta$
  \STATE Set $Q_{1} = (1/N)_{j\in\llbracket 1, N\rrbracket}$
  \FOR{$t=1,..., T$}
  \STATE Observe context $x_{t}$ and probability recommendation $(E_{m}^{(t)})_{m\in\llbracket 1, N\rrbracket}$
  \STATE Pull arm $a_{t}\sim P_{t}$ where $P_{t,j} = \sum_{k=1}^{N} Q_{t,k}E_{j,k}^{(t)}$ 
  \STATE Observe reward $r_{t}$ and define for all arms $i$ $\hat{r}_{t,i} = 1 - \mathds{1}_{\{ a_{t}=i\}}( 1 - r_{t})/P_{t,i}$
  \STATE Define $\tilde{X}_{t,k} = \sum_{a} E_{k, a}^{(t)}\hat{r}_{t,a}$
  \STATE Update $Q_{t+1, j} = \exp(\eta Q_{t,i})/\sum_{j=1}^{N} \exp(\eta Q_{t,j})$ for all experts $i$
  \ENDFOR
\end{algorithmic}
\end{algorithm}

\section{Semi-Online Attacks}

\cite{liu2019data} studies what they call the offline setting for adversarial attacks on stochastic bandits. They consider a setting where a bandit algorithm is successively updated with mini-batch\changebr{es} of fixed size $B$. \changebr{The attacker can tamper with some of the incoming mini-batches. 
 More precisely, they can modify the context, the reward and even the arm that was pulled for any entry of the attacked mini-batches.}
\changebr{The main difference between this type of attacks and the online attacks we considered in the main paper is that we do not assume that we can attack from the start of the learning process: the bandit algorithm may have already converged by the time we attack}. 

We can still study the cumulative cost for the attacker to change the mini-batch in order to fool a bandit algorithm to pull a target arm $a^{\dagger}$ (\changee{here we take $A^{\dagger} = \{ a^{\dagger}\}$}). Contrarily to \cite{liu2019data}, we call this setting semi-online. We first study the impact of an attacker on \linucb where we show that, by modifying only $(K-1)d$ entries from the batch $\mathcal{B}$, the attacker can force \linucb to pull arm $a^{\dagger}$, $M'B - o(M'B)$ times with $M'$ the number of remaining batches updates. The cost of our attack is $\sqrt{MB}$ with $M$ the total number of batches.

\paragraph{Cost of an attack:} If presented with a mini-batch $\mathcal{B}$, with elements $(x_{t}, a_{t}, r_{t})$ composed of the context $x_{t}$ presented at time $t$, the action taken $a_{t}$ and the reward received $r_{t}$, the attacker modifies element $i$, namely  $(x^{i}_{t}, a^{i}_{t}, r^{i}_{t})$ into $(\tilde{x}^{i}_{t}, \tilde{a}^{i}_{t}, \tilde{r}^{i}_{t})$. The cost of doing so is $c^{i}_{t} = ||x^{i}_{t} - \tilde{x}^{i}_{t}||_{2} + \big|\tilde{r}^{i}_{t} - r^{i}_{t}\big| + \mathds{1}_{\{a^{i}_{t} \neq \tilde{a}^{i}_{t}\}}$ and the total cost for mini-batch $\mathcal{B}$ is defined as $c_{\mathcal{B}} = \sum_{i\in \mathcal{B}} c_{t}^{i}$. Finally, we consider the cumulative cost of the attack over $M$ different mini-batches $\mathcal{B}_{1}, \hdots, \mathcal{B}_{M}$, $\sum_{l=1}^{M} c_{\mathcal{B}_{l}}$. The interaction between the environment, the attacker and the learning algorithm is summarized in Alg.~\ref{alg:semi_online_setting}.  
\begin{algorithm}[h]
	\caption{Semi-Online Attack Setting.}
  \label{alg:semi_online_setting}
\begin{algorithmic}
  \STATE {\bfseries Input:} Bandit alg.~$\mathfrak{A}$, size of a mini-batch: $B$
  \STATE Set $t = 0$
  \WHILE{True}
  \STATE $\mathfrak{A}$ observe context $x_{t}$
  \STATE $\mathfrak{A}$ pulls arm $a_{t}$ and observes reward $r_{t}$
  \STATE Interaction $(x_{t}, a_{t}, r_{t})$ is saved in mini-batch $\mathcal{B}$
  \IF{$\big|\mathcal{B}\big| = B$}
  \STATE Attacker modifies mini-batch $\mathcal{B}$ into $\tilde{\mathcal{B}}$
  \STATE Update alg.~$\mathfrak{A}$ with poisoned mini-batch $\tilde{\mathcal{B}}$
  \ENDIF
  \ENDWHILE
\end{algorithmic}
\end{algorithm}


The attack presented here is based on the Ahlberg–Nilson–Varah bound \cite{varah1975lower}, which gives a control on the sup norm of a matrix with dominant diagonal elements. More precisely, when presented with a mini-batch $\mathcal{B}$, the attacker needs to modify the contexts and the rewards. We assume that the attacker knows the number of mini-batch updates $M$ and has access to a lower-bound on the reward of the target arm, $\nu$ as in Assumption~\ref{assumption2}. 

The attacker changes $(K-1)\times d$ rows of the first mini-batch to rewards of $0$ with a context $\delta_a e_i$ for each arm $a \neq a^{\dagger}$ with $(e_i)$ the canonical basis of $\mathbb{R}^{d}$. Moreover, $\delta_{a}$ is chosen such that: 
\begin{equation}
    \delta_{a}  > \max\left(\sqrt{\frac{2MBL^{2}d}{\nu} + dMB},  \sqrt{\frac{4\beta_{max}^2L^{2}d}{\nu^{2}} + dMB}\right)
    \label{eq:delta_botnets}
\end{equation}
with $\beta_{max} = \max_{t=0}^{MB} \beta_a(t)$ and $M$ the number of mini-batch updates.

\begin{prop}\label{prop:attacker_can_choose}
After the first attack, with probability $1-\delta$, \linucb always pulls arm $a^{\dagger}$, 
\end{prop}

\begin{proof}
After having poisoned the first mini-batch $\mathcal{B}$, the latter can be partitioned into two subsets, $\mathcal{B}_{c}$ (with non-perturbed rows) and $\mathcal{B}_{nc}$ (with the poisoned rows). The design matrix of arm $a\neq a^{\dagger}$ for every time $t$ after the poisoning is:
\begin{align}
    V_{t,a} = \lambda I_d +  \sum_{l=1, a_{l} = a}^{t} x_{l}x_{l}^{\intercal} + \delta_{a}^2 \sum_{i=1}^{d} e_i e_i^{\intercal}
\end{align}
For every time $t$, non diagonal elements of $V_{t,a} = (v_{i,j})_{i,j}$ are bounded by: 
\begin{align}\label{non_diagonal_element}
    \forall i, r_i &:= \sum_{j \neq i} v_{i,j}\leq \sum_{j \neq i} \sum_{l=1, a_{l} = a}^{t} \lVert x_l x_{l}^{\intercal}\rVert_\infty\leq dN_{a}(kB)
\end{align}
Whereas for all diagonal elements, $v_{i,i} \geq \delta_{a}^2 > r_i$. Thus $V_{t,a}$ is strictly diagonal dominant and by the Ahlberg–Nilson–Varah bound \cite{varah1975lower}:
\begin{align}\label{eq:bound_norm_design_matrix}
\lVert V_{t,a}^{-1} \rVert_\infty &\leq \frac{1}{\min_{i} \left(\lVert v_{i,i}\rVert - r_i\right)} \leq \frac{1}{\delta_{a}^{2} - dMB}
\end{align}
Then, for every arm $a \neq a^{\dagger}$ and any context $x\in \mathcal{D}$ and any time $t$ after the attack:
\begin{align*}
    \langle\hat{\theta}_{a,t}, x\rangle + \beta_{a}(t) ||x||_{V_{t,a}^{-1}} & \leq  \sum_{l=1, a_{l} = a}^{t} r_t (V_{t,a}^{-1} x_t)^{\intercal} x + \beta_{a}(t)||x||_{1}\sqrt{\lVert V_{t,a}^{-1} \rVert_\infty}\\
    & \leq  \lVert V_{t,a}^{-1} \rVert_\infty dN_{t}(a) \sup_{y\in\mathcal{D}} ||y||_{2}^{2} + \beta_{\text{max}}\sqrt{d}\sup_{y\in\mathcal{D}} ||y||_{2}\sqrt{\lVert V_{t,a}^{-1} \rVert_\infty} < \nu
\end{align*}

    
We have shown that for any arm $a \neq a^{\dagger}$ and any time step $t$ after the attack, the upper confidence bound computed by \linucb is upper-bounded bu $\nu$ the arm $a^{\dagger}$. Then, with probability $1-\delta$, the confidence set for arm $a^{\dagger}$ holds and, for all $x \in \mathcal{D}$, arm $a^{\dagger}$ is chosen by \linucb. The total cost of this attack is $d \sum_{a\neq a^{\dagger}} \delta_{a} L = O(\sqrt{MB})$
\end{proof}

\end{document}